\documentclass[mnsc]{informs3aa} 
\usepackage[final]{showkeys}
\usepackage{subfigure}
\usepackage{arydshln}
\usepackage{makecell}

\OneAndAHalfSpacedXI 

\usepackage{endnotes}

\usepackage{longtable}
\usepackage{tabularx}
\usepackage{booktabs}
\usepackage{etex}
\usepackage{multirow}
\usepackage{array}
\usepackage{bbm}
\usepackage{epstopdf}

\usepackage[title]{appendix}

\newcommand{\vct}{\boldsymbol }

\renewcommand{\tilde}{\widetilde}
\renewcommand{\hat}{\widehat}

\newcommand{\err}{\mathrm{err}}

\renewcommand{\hat}{\widehat}
\renewcommand{\tilde}{\widetilde}

\usepackage{amsbsy, amsfonts, amsgen, amsmath, amsopn, amssymb, amstext,
amsxtra, bezier, color, enumerate, graphicx, latexsym, verbatim,
pictexwd, supertabular, url, dsfont,leftidx, mathrsfs, appendix, setspace}
\usepackage{algorithm}
\usepackage{algorithmicx}
\usepackage{algpseudocode}
\makeatletter
\def\BState{\State\hskip-\ALG@thistlm}
\makeatother

\usepackage{pifont}

\usepackage{natbib}
 \bibpunct[, ]{(}{)}{,}{a}{}{,}%
 \def\bibfont{\small}%
\usepackage[colorlinks=true,bookmarks=false,urlcolor=blue, citecolor=blue,linkcolor=blue,bookmarksopen=false,draft=false]{hyperref}

\newcommand{\ud}{\mathrm d}

\usepackage{parskip} 

\newcommand{\sgn}{\mathrm{sgn}}

\TheoremsNumberedThrough     
\ECRepeatTheorems

\EquationsNumberedThrough    

\MANUSCRIPTNO{} 

\definecolor{DSgray}{cmyk}{0,1,0,0}

\begin{document}


\RUNAUTHOR{Chen, Liu and Wang}

\RUNTITLE{Actively Learning for Contextual Search}

\TITLE{Active Learning for Contextual Search with Binary Feedback}

\ARTICLEAUTHORS{%
\AUTHOR{Xi Chen \thanks{Author names listed in alphabetical order.}}
\AFF{Leonard N.~Stern School of Business, New York University, New York, NY 10012, \EMAIL{xc13@stern.nyu.edu}}
\AUTHOR{Quanquan Liu }
\AFF{Naveen Jindal School of Management, University of Texas at Dallas, Richardson, TX 75080,   \EMAIL{qxl220001@utdallas.edu}}
\AUTHOR{Yining Wang}
\AFF{Naveen Jindal School of Management, University of Texas at Dallas, Richardson, TX 75080,  \EMAIL{yxw220006@utdallas.edu}}
} 

\ABSTRACT{In this paper, we study the learning problem in contextual search, which is motivated by applications such as crowdsourcing and personalized medicine experiments. In particular, for a sequence of arriving context vectors, with each context associated with an underlying value, the decision-maker either makes a query at a certain point or skips the context. The decision-maker will only observe the binary feedback on the relationship between the query point and the value associated with the context. We study a PAC learning setting, where the goal is to learn the underlying mean value function in context with a minimum number of queries. To address this challenge, we propose a tri-section search approach combined with a margin-based active learning method. We show that the algorithm only needs to make $\tilde O(1/\varepsilon^2)$ queries to achieve an $\varepsilon$-estimation accuracy. This sample complexity significantly reduces the required sample complexity in the passive setting where neither sample skipping nor query selection is allowed, which is at least $\Omega(1/\varepsilon^3)$.
}



\date{}

\maketitle

\section{Introduction}





Contextual search, which extends the classical binary search problem to high dimensions, finds a wide range of applications, such as crowdsourcing and personalized medicine. In the contextual search problem, for each round $i=1,2,3,\cdots$, an item (e.g., a customer or a patient) arrives sequentially, each with a contextual vector $x_i\in\mathbb R^d$ accessible to the decision-maker. We assume that the context $x_i$ incurs an unknown stochastic \emph{value} $u_i=v(x_i) + \xi_i$, where $v(x_i)$ is the mean value function of $x_i$ and $\xi_i$ is the stochastic noise.  The decision-maker selects a query $b_i \in \mathbb{R}$ and then observes the \emph{binary feedback}, i.e., whether $u_i \geq b_i$ or vice versa. The true value $u_i$ will never be revealed. To better fit our motivating applications illustrated below, the decision-maker is allowed to skip making a query on certain contextual vectors to save her budget. Our goal is to learn the mean value function $v(x_i)$ with a minimum number of queries/trials. It is worth noting that we adopt the sample complexity as the objective instead of revenue/cost since we focus on the experimental phase for the learning purpose. In this phase,  the number of trials is usually quite small and thus it is common to treat cost equally for each trial.  We now briefly describe two motivating applications:

\noindent \textbf{Personalized Medicine Experiment:} Let us consider the example of clinical trials, where the goal of an experiment is to determine the proper dosage $v(x)$ in radiation therapy.   The profile of each potential experimental unit is characterized by $x$ (e.g., her demographics, diagnosis, medications, and genetics).  \citet{bastani2020online} adopted a linear bandit model (i.e., the linear form of $v(x)$) to investigate the relationship between the optimal dosage and the patients' profile. In this experiment, the lab posts an advertisement to the public and receives nominations (e.g., someone will call the lab to express her interest). After the lab receives a nomination and conducts some pre-screening to collect the profile information $x$, the lab can simply reject a nomination without further experimental procedure. For example, if the lab has already experimented on a similar unit (i.e., a unit with a similar profile $x$),  the lab will naturally reject this potential unit. If the lab decides to accept the experimental unit, we assume that the lab will recommend a dosage $b_i$ and receive the binary feedback on whether the recommended dosage is above or below the appropriate level. As performing a radiation therapy experiment is costly and time-consuming, a common goal is to use the minimum number of trials to learn the ideal personalized dosage level function (i.e., the $v(\cdot)$ function).

\noindent \textbf{Crowdsourcing:} 	In a crowdsourcing experiment, the decision-maker hires a crowdsourced expert to help determine the difficulty level (e.g., measured by completion time) of different tasks characterized by their context vectors $x$. Assuming for each task $i$, the underlying difficulty level is $u_i= v(x_i)+ \xi_i$.  Numerous psychology studies have shown one is more good at providing pairwise comparison than absolute numerical estimate \citep{Shiffrin:94, Stewart:05}. Therefore, instead of asking the expert to provide a numerical estimate of the difficulty level, the decision-maker will give an estimate $b_i$. Then the expert provides binary feedback on whether she believes $b_i$ is an over-estimate or under-estimate. In such a crowdsourcing experiment, it is natural that  the decision-maker will not bother the expert to provide feedback for some tasks (e.g., those tasks similar to previously queried jobs).

Motivated by these applications, the goal of this paper is to propose an efficient algorithm to learn $v(x)$. Following the existing literature on contextual search and feature-based pricing, we also adopt a
linear model of the mean valuation function, i.e., $v(x)= \langle x, w^* \rangle - \mu^*$ for some unknown coefficient vector $w^* \in \mathbb{R}^d$ and the intercept $\mu^* \in \mathbb{R}$. As compared to the existing literature, our contextual search problem has the following unique features, which calls for new algorithmic development:
\begin{enumerate}
	\item First,  the existing contextual search setup aims to minimize either the absolute loss $|b_i- v(x_i)|$ or the $\epsilon$-ball loss $\mathbb{I}(|b_i- v(x_i)|>\epsilon)$ for some pre-determined $\epsilon$ over time. Here $\mathbb{I}(\cdot)$ denotes the indicator function. In contrast, we consider a learning problem, where the goal is to learn $v(x)$ as accurately as possible. Therefore, we adopt a probably approximately correct (PAC) setting (see \eqref{eq:defn-pac} in Sec.~\ref{sec:model}) instead of  regret minimization setting in existing literature \citep{Lobel:17:Multidimensional,Leme:18:contextual,Cohen:20:feature,Krishnamurthy21Contextual}. To facilitate the analysis of this learning problem, we assume the stochasticity of the contextual information $x_i$.  
	
	\item Second, as we are motivated by \emph{experimental} applications, the decision-maker should judge the benefit of a context $x_i$ to the learning problem. Therefore,  compared to the existing contextual search, our problem has another layer of decision, i.e., whether to conduct a query or not, beyond the decision of the query point itself. 
\end{enumerate}

To address this problem, we adopt the active learning framework from machine learning research \citep{Settles:12:active}. In particular, we adopt the margin-based active learning approach \citep{balcan2007margin}. At a high level, let  $\hat v(\cdot)$ be the current estimate of the underlying $v(\cdot)$ function {and $\hat b$ be the query point}. For an arriving context $x$, the margin-based active learning will make a query if $|\hat v(x) - \hat b|$ is sufficiently small, which indicates that it is difficult to determine the relationship between $\hat b$ and $\hat v(x)$. Although it is an intuitive approach, existing margin-based active learning approaches cannot be directly applied to address our problem due to the existence of the intercept $\mu^*$. In fact, a famous negative result by \cite{dasgupta2005coarse} shows that active learning cannot significantly improve sample complexity over \emph{passive learning} for linear binary classification models \emph{with intercepts} in its most general form. It is worth noting that throughout this paper, by ``passive learning'' we refer to the learning paradigm in which the decision making can neither skip samples (regardless of their contextual information $x$) nor adaptively change actions/queries.
Please refer to Figure \ref{fig:negative-example} in Sec.~\ref{sec:related} for details.

To address this challenge, we propose an active learning procedure consisting of three major stages:
\begin{enumerate}
\item The first stage of the algorithm is to use trisection search to locate two queries $\hat b_1$ and $\hat b_2$ that are close to the underlying intercept term $\mu^*$,
without consuming too many labeled (queried) samples.
In this first stage sample selection (i.e., determining whether a sample is to be labeled/queried or not) is \emph{not} carried out,
but the algorithm will actively explore different actions in order to obtain $\hat b_1,\hat b_2$ that are close to $\mu^*$;
\item The second stage of the algorithm is to apply margin based active learning to learn the linear model $w^*$ and an intercept term depending on both $\mu^*$ and $\hat b_1,\hat b_2$.
In this second stage sample selection will be carried out, as only those users with contextual vectors $x_t$ close to classification hyperplanes will be 
queried/labeled (see Algorithm \ref{alg:margin-based-active-learning} later for details). 
The actions taken in this stage (on selected samples) will be fixed to either $\hat b_1$ or $\hat b_2$ obtained in the first stage.

Note that, although this classification model still has non-zero intercept terms, the closeness of $\hat b_1,\hat b_2$ to $\mu^*$ would imply
that the obtained labels under actions $\hat b_1$ or $\hat b_2$ are \emph{balanced}, circumventing the negative results in the work of \cite{dasgupta2005coarse}
which specifically constructed counter-examples with unbalanced labels.
In Figure \ref{fig:negative-example} and the following related work section we give a detailed account of this negative example
and how it presents challenges to active learning.
Indeed, our theoretical analysis extends the arguments in \cite{balcan2007margin} to this more general setting of linear classification
\emph{with intercepts} and balanced labels, with similar convergence rates derived.

\item The final stage of the algorithm is to reconstruct the mean utility model $\hat v(\cdot)$ from the estimated linear model and intercepts.
Because margin-based active learning can only estimate a linear model up to scales, we need model estimates at two different actions $\hat b_1,\hat b_2$
(corresponding to two different effective intercepts) in order to reconstruct $w^*$ and $\mu^*$ in $v(\cdot)$.
Details of how this reconstruction is carried out are given in the last two lines of Algorithm \ref{alg:meta-algorithm}.
\end{enumerate}

We establish the sample complexity bound for the proposed margin-based active learning with a tri-section search scheme. We assume that with $\tilde{O}(1/\varepsilon^3)$ total number of incoming contexts, the decision-maker only needs to make $\tilde{O}(1/\varepsilon^2)$ queries to estimate the mean value function $v(x)$ within $\varepsilon$-precision (with high probability). Here $\tilde{O}$ here hides the dependence on $d$ and other logarithmic factors. We also show that in the passive setting, where the decision-maker is required to conduct queries for all arriving contexts as in the standard contextual search, the sample complexity would be at least $\Omega(1/\varepsilon^3)$ (see Remark \ref{rem:lower}).


\subsection{Related work}
\label{sec:related}

Our problem setting can be viewed as a variant of the contextual search problem, which is an extension of the classical binary search. In binary search, the decision-maker tries to guess a fixed constant $\mu^*$ (i.e., the value $u_i\equiv \mu^*$ for all $i$ in our problem) by iteratively making queries $b_i$. In the PAC learning setting, the binary search algorithm only needs $O(\log(1/\varepsilon))$ queries to estimate $\mu^*$ within $\varepsilon$-precision. Due to the importance of applications such as personalized medicine and feature-based pricing, contextual search has received a lot of attention in recent years. The existing literature mainly adopts the linear model for the mean value function. For $\varepsilon$-ball loss $\sum_{i} \mathbb{I}(|b_i-v(x_i)|>\varepsilon)$,  \citet{Lobel:17:Multidimensional} established the $\Omega(d \log (1/\varepsilon \sqrt{d}))$ regret lower bound and proposed the project volume algorithm that achieves a near-optimal regret of $O(d \log(d/\varepsilon))$. For absolute loss 
$\sum_{i} |b_i-v(x_i)|$, \cite{Leme:18:contextual} established the regret bound of $O(\text{poly}(d))$. As we explained in the introduction, to fit the applications considered in our paper, we adopt a PAC learning setting and equip the decision-maker with the ability to pass an incoming context. {While most contextual search settings in the literature consider adversarial contextual information, we assume the stochasticity of the contextual information as we study a learning problem.}

\begin{figure}[t]
\centering
\includegraphics[width=0.45\textwidth]{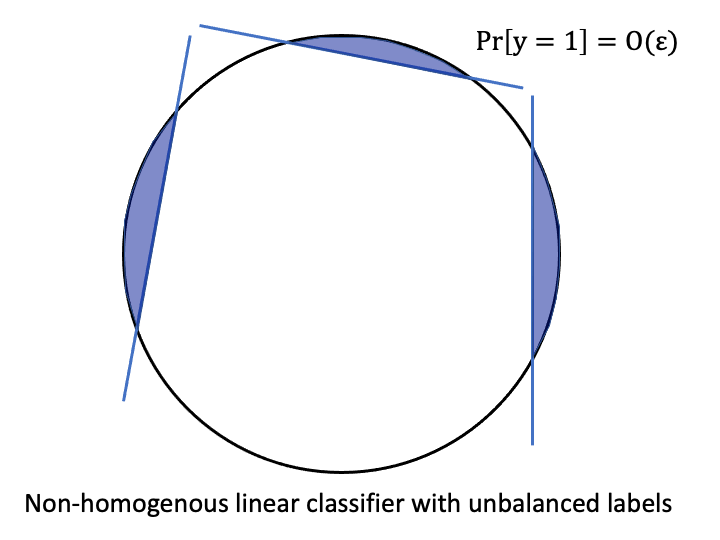}
\includegraphics[width=0.45\textwidth]{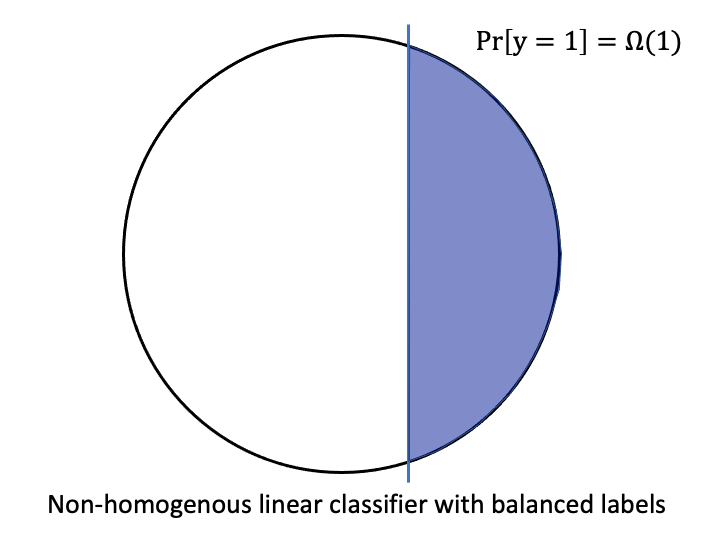}
\caption{Illustration of negative examples of problem instances constructed in \citep{dasgupta2005coarse}.}
{\small Note: the left panel shows examples of non-homogeneous linear classifiers with \emph{unbalanced} labels,
for which \cite{dasgupta2005coarse} shows that active learning (i.e., sample selection for labeling purposes) cannot lead to
significantly improved sample complexity. The right panel shows non-homogenous linear classifiers with \emph{balanced} labels,
for which improvements of sample complexity can be made via sample selection (active learning).}
\label{fig:negative-example}
\end{figure}

Active learning is an important research area in machine learning, originating from the seminal work of \cite{cohn1994improving} dating
back to the 1990s.
The main idea behind active learning is to equip the learning algorithm with the ability to \emph{select} samples or data points to be labeled,
improving its sample complexity in applications where labels are expensive to obtain but unlabeled data are abundant.
There have been many successful algorithms developed for active learning, such as bisection search for one-dimensional noiseless problems
\citep{dasgupta2005coarse}, greedy method \citep{dasgupta2005analysis}, disagreement-based active learning \citep{hanneke2007bound,balcan2009agnostic,zhang2014beyond}, margin based active learning \citep{balcan2007margin,balcan2013active,wang2016noise}
and active learning based on surrogate loss functions \citep{awasthi2017power,balcan2017sample}.
Due to the vast literature on active learning we cannot cite all related works here, and would like to refer interested readers
to the excellent review of \cite{hanneke2014theory} for an overview of this area.

Our approach in this paper resembles the margin-based active learning method \citep{balcan2007margin,balcan2013active,wang2016noise}
which is developed for linear classifiers and have been popular in the active learning literature due to its intuitive nature, tight sample complexity,
and relative ease of implementation.
However, while linear classifiers seem simple, non-homogeneous linear classifiers (i.e., linear classifiers with \emph{an intercept term}) present notorious challenges
to active learning algorithms.
More specifically, the work of \cite{dasgupta2005analysis} shows that when $d\geq 2$ and non-homogeneous linear classifiers produce unbalanced samples,
such as the example shown on the left panel of Figure \ref{fig:negative-example}. In this illustrative example,  potential linear classifiers are within $O(\varepsilon)$ distance to the 
domain boundary, and thus, active learning \emph{cannot} asymptotically improve sample complexity over passive learning as it takes $O(1/\varepsilon)$ samples to 
hit the boundaries.
Note that, it is easy to verify that, if a non-homogenous linear classifier is within $\epsilon$ distance to the boundary and the underlying distribution
of unlabeled samples is relatively uniform, the probability of seeing a positive sample (as indicated in the region colored by blue in Figure \ref{fig:negative-example}) is also on the order of $O(\varepsilon)$.
To overcome this counter-example, in this paper we exploit the special structure in the contextual search problem
to ``balance'' the labels, as shown on the right panel of Figure \ref{fig:negative-example}.
While the balanced model still possesses a non-zero intercept term, the classifier will be generally $\Omega(1)$ away from the boundary,
which our theoretical analysis shows is sufficient of obtaining desired sample complexity results for active learning.

{
It is also interesting to make comparisons to other margin-based active learning work. For example, the work by \cite{awasthi2017power} considers the following setting: for an underlying (unknown) model $w^*$
and feature vector $x$, the algorithm observes $\mathrm{sgn}(\langle x,w^*\rangle)$ with probability $1-\eta$
and an adaptively chosen label with probability $\eta$.
In comparison, in our problem setting the algorithm observes labels with probability related to the margin $|\langle x,w^*\rangle|$. Such a difference in the setup leads to a fundamental difference in the sample complexity: the sample complexity in \cite{awasthi2017power} is poly-logarithmic in $1/\varepsilon$,
while in our problem setting a $\mathrm{poly}(1/\varepsilon)$ sample complexity is necessary. In fact, $\log(1/\varepsilon)$ sample complexity is only possible if one has deterministic labels
or probabilistic labels satisfying the Massart noise condition; that is, for any $x\in\mathbb R^d$ with $\langle x, w^*\rangle > 0$,
$\Pr[y=1|x] > 1/2 + c$ for some constant $c>0$ (and vice versa for all $\langle x, w^*\rangle < 0$).
Such a condition clearly is not satisfied by the setting studied in this paper, in which $\Pr[y=1|x]\to 1/2$ as $\langle x,w^* \rangle\to 0^+$.
For noise distributions not satisfying the Massart condition, $\mathrm{poly}(1/\varepsilon)$ samples are necessary
(see, e.g., \cite{ben2014sample,balcan2007margin,wang2016noise}).
}

{
Our work is also related to the pure-exploration problem and sequential experimental design \citep{elfving1952optimum,chernoff1959sequential,albert1961sequential,naghshvar2013active,wang2020adaptive, feng2021robust,wager2021diffusion, araman2021diffusion,li2017optimal,Chen:22:asym}. 	In experimental design problems, the decision maker is capable of choosing the context vector $x$.
However, in application settings considered in this paper (e.g., experimental units arriving sequentially), it is impractical to assume that the context vectors $\{x_t\}_{t=1}^T$ could be chosen arbitrarily. Thus, we only allow the decision maker to decide whether to skip a query.
}

Active learning has been an important area in machine learning. However, it has not received a lot of attention in operations management. This paper takes a preliminary step on exploring the applications of active learning, and hopefully, it will inspire more research on active learning to address challenges arising from operations management.

\subsection{Paper organization and notations}

The rest of the paper is organized as follows. Sec.~  \ref{sec:model} describes the problem formulation and necessary assumptions. Sec.~\ref{sec:algo} develops our  margin-based active learning algorithm with the tri-section search and establishes the sample complexity bound. The technical proofs are provided in Sec.~\ref{sec:proof}. We provide the numerical simulation studies in Sec.~\ref{sec:numerical}, followed by the conclusion in Sec.~\ref{sec:con}. Proofs of some technical lemmas are relegated to the appendix.

{
In our paper the asymptotic is with respect to $d$ and $\varepsilon$, with all other parameters being functions of $d,\varepsilon$
	and other problem-dependent constants (e.g., $B,c_x,C_x,c_\xi,C_\xi$) that do not change with $d,\varepsilon$.
	We say that $f(x)=O(g(x))$ if there exist constants $d_0,\varepsilon_0$ and $C<\infty$ such that for all $d\geq d_0$ and $\varepsilon\leq\varepsilon_0$,
	$f(x)\leq Cg(x)$. If we omit dependency on constants $\theta=(B,c_x,C_x,c_\xi,C_\xi)$ in the big-O notation, then the constant $C$ can be a function of $\theta$.
	If we further omit poly-logarithmic dependency (by using the notation $\tilde O(g(x))$), then the constant $C$ can depend on $\log^c x$ for some constant $c$.
	}

\section{Problem Formulation and Assumptions}
\label{sec:model}

In our modeling, assuming the items (e.g., ads or experimental units)  $i=1,2,3,\cdots$ arrive sequentially, each with a contextual or feature vector $x_i\in\mathbb R^d$
accessible to the decision-maker.  We assume that the contextual vectors $\{x_i\}_{i\geq 1}$ are independently and identically distributed with respect to an unknown underlying distribution $P_X$. We also assume that 
{$\|x_i\|_2 \leq 1$ for the ease of illustration.}
Given the contextual vector $x_i\in\mathbb R^d$, the ``valuation'' of the item (e.g.,  the appropriate dosage in personalized medical treatment)
follows a linear model:
\begin{equation}
u_i = v(x_i) + \xi_i = \langle x_i, w^*\rangle - \mu^* + \xi_i,
\label{eq:defn-v}
\end{equation}
where $v(\cdot)=\langle\cdot,w^*\rangle - \mu^*$ is an underlying linear model with a fixed but unknown coefficient vector  $w^*\in\mathbb R^d$, the intercept $\mu^*\in\mathbb R$,
and the noise $\{\xi_i\}_{i\geq 1}$, which are independently and identically distributed stochastic variations with respect to an unknown distribution $P_\xi$.

After observing the contextual vector $x_i\in\mathbb R^d$, the decision-maker will do either one of the following :
\begin{enumerate}
\item Let the item pass without taking any actions, and thereby without obtaining any feedback/information;
\item Make a query at $b_i\in\mathbb R$, and observe the binary feedback $y_i=1$ if $u_i\geq b_i$ or $y_i=-1$ if $u_i<b_i$.
\end{enumerate}
Since making a query (e.g., admitting an experimental unit into a clinical trail program) incurs much higher implicit
cost as compared to passing (i.e., taking no action), the main goal of the decision-maker is to use as few number of queries as possible to estimate the mean valuation function $v(\cdot)$ to a certain precision. More specifically, let $\varepsilon,\delta\in(0,1)$ be target accuracy and probability parameters. 
We use $n(\varepsilon,\delta)$ to denote the number of queries a learning algorithm takes in order to produce an estimate $\hat v(\cdot)$ that satisfies
\begin{equation}
\sup_{\|x\|_2\leq 1} \big|\hat v(x)-v(x)\big| \leq \varepsilon, \;\;\;\;\;\;\text{with probability $\geq 1-\delta$}.
\label{eq:defn-pac}
\end{equation}
Clearly, the smaller $n(\varepsilon,\delta)$ is the more efficient the designed learning algorithm is.
The main objective of this paper is to design an active learning algorithm that minimizes $n(\varepsilon,\delta)$.
Additionally, we use $m(\varepsilon,\delta)$ to denote the number of \emph{total} samples (i.e., the number of total incoming contexts) an algorithm requires to 
obtain an estimate $\hat v$ satisfying Eq.~(\ref{eq:defn-pac}).
While those incoming contexts skipped by our algorithm usually do not  incur extra cost,  it is desirable that  $m(\varepsilon,\delta)$ is reasonable
because the supply of experimental units might still be limited.
In active learning literature, an $m(\varepsilon,\delta)$ is \emph{reasonable} if it is a polynomial function in terms of  $1/\varepsilon,\log(1/\delta)$ and $d$ \citep{cohn1996neural,cohn1994improving,balcan2007margin}.

Throughout this paper we impose the following assumptions.
\begin{enumerate}
\item[(A1)] There exists a constant $B<\infty$ such that $\|w^*\|_2\leq B$ and $|\mu^*|\leq B$;
\item[(A2)] The distribution $P_X$ satisfies the following condition: it is supported on the unit $\ell_2$ ball $\mathbb B_2(d) = \{x\in\mathbb R^d: \|x\|_2\leq 1\}$;
it admits a probability density function $f_x(\cdot)$; there exist constants $0<c_x\leq C_x<\infty$ such that $c_x f_u(x)\leq f_x(x)\leq C_xf_u(x)$ for all $x\in\mathbb B_2(d)$, where $f_u$ is the probability density function (PDF) of the uniform distribution on $\mathbb B_2(d)$;
\item[(A3)] The distribution $P_\xi$ satisfies the following condition: $\Pr[\xi\leq 0]=\Pr[\xi\geq 0] = 1/2$; 
it admits a probability density function $f_\xi(\cdot)$; there exist constants $0<c_\xi\leq C_\xi<\infty$ such that 
{$\sup_{\xi\in\mathbb R}f_\xi(\xi)\leq C_\xi/\|w^*\|_2$ and
$\inf_{|\xi|\leq 2}f_\xi(\xi)\geq c_\xi/\|w^*\|_2$.
}
\end{enumerate}

Assumption (A1) is a standard bounded assumption imposed on model parameters.
Assumption (A2) assumes that the contextual vectors are independently and identically distributed, with respect to a bounded
and non-degenerate distribution $P_X$ that is unknown.
Similar ``non-degenerate'' or ``covariate diversity'' assumptions were also adopted in the contextual learning literature \citep{bastani2020online,bastani2021mostly},
and the assumption is actually weaker than some of the existing works on active learning \citep{balcan2007margin,wang2016noise}, which requires $P_X$ to be 
the exact uniform distribution over $\mathbb B_2(d)$.

Assumption (A3) is a general condition imposed on the distribution $P_\xi$ of the noise variables.
Essentially, it assumes that zero is the median of the noise distribution $P_\xi$, which ensures that the linear classifier is the optimal Bayes classifier.
The same assumption is common in the active learning literature \citep{balcan2007margin,wang2016noise}.
Note that we do not assume the noise distribution $P_\xi$ has any specific parametric forms (e.g., Logistic or Probit noises), 
making it generally applicable to a broad range of problems.
{Note also that Assumption (A3) requires the noise distribution to scale together with $\|w^*\|_2$ in order to preserve signal-to-noise
ratios. In the case of a signal-independent assumption $\sup_{\xi\in\mathbb R}f_\xi(\xi)\leq C_\xi'$ and $\inf_{|\xi|\leq 2}f_\xi(\xi)\geq c_\xi'$,
the change-of-parameter $C_\xi'=C_\xi/\|w^*\|_2$ and $c_\xi'=c_\xi/\|w^*\|_2$ can be used to bring the signal level $\|w^*\|_2$ into the sample complexity analysis.
}


\section{Margin-based Active Learning with Tri-section Search}
\label{sec:algo}

\begin{algorithm}[t]
\caption{A meta-algorithm for actively learning contextual functions.}
\label{alg:meta-algorithm}
\begin{algorithmic}[1]
\State \textbf{Input}: dimension $d$, accuracy parameters $\varepsilon,\delta$, algorithm parameters $\kappa_m,\kappa_n,\kappa_\varepsilon,\beta_0$.
\State $\hat b_1,\hat b_2\gets \textsc{TrisectionSearch}(\varepsilon_s,\delta_s)$ with $\varepsilon_s=0.1/\sqrt{d-1}$, $\delta_s=\delta/3$;
\State Let $\varepsilon_a=\kappa_\varepsilon{\varepsilon^2}/\ln^2(1/\varepsilon)$, $\delta_a=\delta/3$;
\State $(\hat w_1,\hat\beta_1)\gets \textsc{MarginBasedActiveLearning}(\hat b_1,\varepsilon_a,\delta_a,\kappa_m,\kappa_n,\sqrt{\varepsilon_a},\beta_0)$;
\State $(\hat w_2,\hat\beta_2)\gets \textsc{MarginBasedActiveLearning}(\hat b_2,\varepsilon_a,\delta_a,\kappa_m,\kappa_n,\sqrt{\varepsilon_a},\beta_0)$;
\State Let $\hat\alpha = (\hat b_2-\hat b_1)/(\hat\beta_2-\hat\beta_1)$;\label{line:defn-hat-alpha}
\State \textbf{Output}: {utility function estimate $\hat v(\cdot)=\langle\cdot,\hat w\rangle - \hat\mu$, where $\hat w=\hat\alpha\hat w_1$\label{line:defn-hat-v}
and $\hat\mu = \hat\alpha\hat\beta_1-\hat b_1$.}
\end{algorithmic}
\end{algorithm}

The main algorithm we proposed for actively learning contextual functions is given in Algorithm \ref{alg:meta-algorithm}.
The main idea of the proposed algorithm can be summarized as follows.

The first step is to find two actions $\hat b_1,\hat b_2$ that are reasonably close to the mean utility $\mu^*$.
This is to ensure that when the actions are fixed at $\hat b_1$ or $\hat b_2$, the labels received from user streams are relatively balanced,
thereby circumventing the negative results in the work of \cite{dasgupta2005coarse}.
In Sec.~\ref{sec:alg-bisection} we show how $\hat b_1,\hat b_2$ can be found without using too many labeled samples,
by using a trisection search idea.

After we obtained candidate actions $\hat b_1$ and $\hat b_2$, we use a margin-based active learning algorithm
to estimate the linear model $w^*$ and mean utility $\mu^*$.
The margin-based active learning algorithm is similar to the work of \cite{balcan2007margin}, with the difference being that in our setting
the active learning algorithm needs to incorporate a (relatively small) intercept term, which complicates its design and analysis.

Finally, we use the estimates $(\hat w_1,\hat\beta_1)$ and $(\hat w_2,\hat\beta_2)$ obtained from the above-mentioned active learning procedure
\emph{under two different fixed actions $\hat b_1,\hat b_2$} to reconstruct the linear utility parameters $w^*$ and $\mu^*$.
The reason we need two fixed actions $\hat b_1,\hat b_2$ is because the active learning procedure solves a classification problem,
for which we can only estimate the linear model and its intercept \emph{up to scalings} because if one multiplies both the linear model and its intercept
by a constant the resulting classification problem is the same.
Hence, we need two fixed actions $\hat b_1,\hat b_2$ to construct an approximate linear system of equations,
the solution of which would give us consistent estimates of $w^*$ and $\mu^*$.

Below we briefly explain our intuition behind the construction of the utility function estimate $\hat v(\cdot)$ in Algorithm \ref{alg:meta-algorithm}.
For simplicity we will omit the learning errors that occurred in the two \textsc{MarginBasedActiveLearing} invocations.
Because the margin based active learning algorithm learns linear classifiers up to normalization (see Algorithm \ref{alg:margin-based-active-learning}), we have the following equivalence:
\begin{eqnarray*}	
 \langle \hat w_1, x \rangle - \hat \beta_1 >0 &  \Longleftrightarrow &  \langle w^*, x \rangle - \mu^* > \hat  b_1; \\
 \langle \hat w_2, x \rangle - \hat \beta_2 >0 & \Longleftrightarrow & \langle w^*, x \rangle - \mu^* > \hat b_2,
\end{eqnarray*}
where $\|\hat w_1\|_2=\|\hat w_2\|_2=1$ due to the construction of Algorithm \ref{alg:margin-based-active-learning}.
Again, we emphasize that the above equivalence only holds approximately  due to learning errors of $\hat w_1,\hat\beta_1,\hat w_2,\hat\beta_2$,
but we will omit these learning errors for ease of explanation.
Let $\alpha=\|w^*\|_2$. We have $ \hat \beta_1 =(\mu^* +  \hat  b_1)/\alpha$ and  $ \hat \beta_2 =(\mu^* +  \hat  b_2)/\alpha$. Therefore, we set $\hat \alpha= (\hat b_2-\hat b_1)/(\hat\mu_2-\hat\mu_1)$ as the estimate of $\alpha$, and $\hat \mu = \hat \alpha  \hat \beta_1 -  \hat  b_1$ as the estimate of $\hat \mu$. Thus, we  obtain the utility function estimate $\hat v(\cdot)$ in Algorithm \ref{alg:meta-algorithm}.

\subsection{Tri-section search for accurate mean utility}\label{sec:alg-bisection}

\begin{algorithm}[t]
\caption{A tri-section search algorithm to roughly estimate the mean utility parameter $\mu^*$}
\label{alg:bisection-search}
\begin{algorithmic}[1]
\Function{TrisectionSearch}{$\varepsilon_s,\delta_s$}
	\State Initialize: $n=0$, lower and upper bounds $\hat b_1=-B$, $\hat b_2=B$;
	\While{$\hat b_2-\hat b_1>\varepsilon_s$}
		\State $\hat b_3\gets\hat b_1+(\hat b_2-\hat b_1)/3$, $\hat b_4\gets\hat b_2-(\hat b_2-\hat b_1)/3$,
		$\hat n=\hat r_3=\hat r_4=0$, $\underline p_3=\underline p_4=0$, $\overline p_3=\overline p_4=1$;
		\While{$\underline p_3\leq 0.5\leq\overline p_3$ and $\underline p_4\leq 0.5\leq \overline p_4$}
			\State For an incoming user $x$, take action $\hat b_3$ and observe result $y\in\{0,1\}$;
			\State For another incoming user $x'$, take action $\hat b_4$ and observe result $y'\in\{0,1\}$;
			\State $n\gets n+1$, $\hat n\gets\hat n+1$, $\hat r_3=\hat r_3+\vct 1\{y_3=1\}$, $\hat r_4=\hat r_4+\vct 1\{y_4=1\}$;
			\State Update: $[\underline p_3,\overline p_3]\gets \frac{\hat r_3}{\hat n} \pm \sqrt{\frac{\ln(8n^2/\delta_s)}{2\hat n}}]$
			and $[\underline p_4,\overline p_4]\gets \frac{\hat r_4}{\hat n} \pm \sqrt{\frac{\ln(8n^2/\delta_s)}{2\hat n}}]$;
		\EndWhile
		\State Set $\hat b_1\gets\hat b_3$ if $\underline p_3>0.5$ or $\underline p_4>0.5$ and $\hat b_2\gets\hat b_4$ otherwise; 
	\EndWhile
	\State \textbf{return} $(\hat b_1,\hat b_2)$.
\EndFunction
\end{algorithmic}
\end{algorithm}

Let $b^*\in\mathbb R$ be the unique value such that $\Pr_{x\sim P_X}[v(x)\geq b^*] = 1/2$.
Because $P_X$ and $P_\xi$ have PDFs, such a value of $b^*$ exists and is unique. 
Intuitively, if one commits to the fixed action $b^*$ then the labels received by the algorithm should be balanced.
Algorithm \ref{alg:bisection-search} shows how to find actions $\hat b_1,\hat b_2$ that are reasonably close to $b^*$,
without consuming too many labeled samples.

\begin{figure}[t]
\centering
\includegraphics[width=0.45\textwidth]{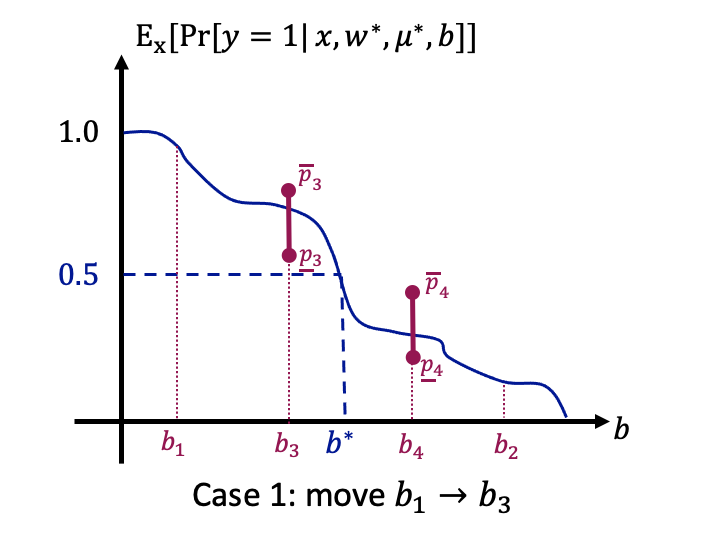}
\includegraphics[width=0.45\textwidth]{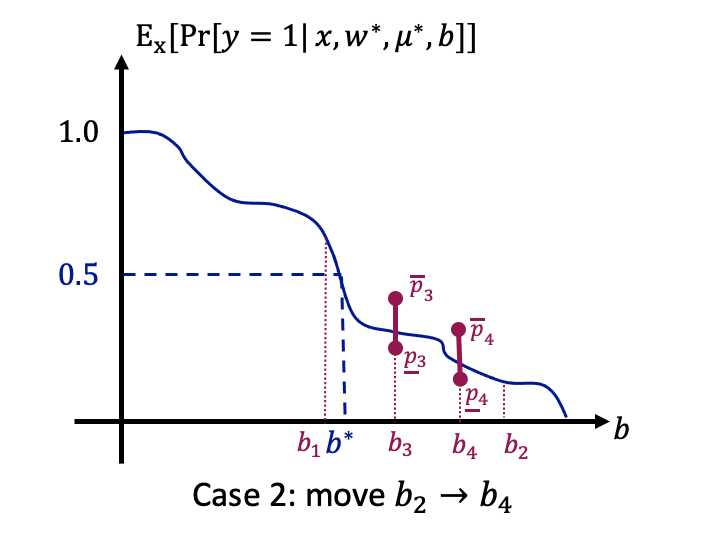}
\caption{Graphical illustration of the main idea behind Algorithm \ref{alg:bisection-search}.}
{\small Note: the left panel shows the first case of the trisection search, in which $\wp(b_3)=\mathbb E_{x\sim P_X}[y|x,w^*,\mu^*,b_3] > 1/2$.
Once $\underline p_3$ exceeds $1/2$, the algorithm will move $b_1$ to $b_3$.
The right panel shows the second case of the trisection search, in which $\wp(b_4)=\mathbb E_{x\sim P_X}[y|x,w^*,\mu^*,b_4]<1/2$.
As both $\overline p_3,\overline p_4$ are below 1/2, the algorithm will move $b_2$ to $b_4$.
The strict monotonicity of $\wp(\cdot)$ as a function of $b$ ensures that the trisection search will never exclude $b^*$ from $[b_1,b_4]$,
and that the search will terminate in $O(d \log(d/\delta))$ iterations (see Lemma \ref{lem:key-bisection-search}).}
\label{fig:illustration-trisection}
\end{figure}

The main idea behind Algorithm \ref{alg:bisection-search} is a trisection search approach, motivated by the fact that
the probability $\Pr_{x\sim P_X}[v(x)\geq b]$ is a monotonically decreasing function of $b$, and furthermore as $|b-b^*|$ increases
the gap between $\Pr_{x\sim P_X}[v(x)\geq b]$ and $\Pr_{x\sim P_X}[v(x)\geq b^*]=1/2$ will also increase (see, e.g., Lemma \ref{lem:beta-deviation} in the proof).
This allows us to use a trisection search procedure to localize the value of $b^*$, by simply comparing an empirical estimate of 
$\Pr_{x\in P_X}[v(x)\geq b]$ at the current value of $b$.
More specifically, at an iteration $\hat b_3,\hat b_4$ are the two midpoints and $[\underline p_3,\overline p_3]$ are lower and upper estimates of 
$\Pr_{x\sim P_X}[v(x)\geq \hat b_3]$ and similarly $[\underline p_4,\overline p_4]$ are lower and upper estimates for $\Pr_{x\sim P_X}[v(x)\geq \hat b_4]$.
With either probability being separated from $1/2$, the algorithm could move $\hat b_1$ or $\hat b_2$ to $\hat b_3$ or $\hat b_4$. The algorithm is guaranteed to maintain that $b^*\in[\hat p_1,\hat p_2]$, thanks to the monotonicity of $\Pr_{x\sim P_X}[v(x)\geq b]$ with respect to $b$.

The following technical lemmas are the main results explaining the objective and guarantee of Algorithm \ref{alg:bisection-search},
which are proved in Sec.~\ref{sec:proof-bisection}.
\begin{lemma}
Suppose $d\geq 2$ and let $\beta^*=b^*+\mu^*$. Then $\frac{|\beta^*|}{\|w^*\|_2}\leq \sqrt{\frac{2\ln(100C_xC_\xi/c_xc_\xi)}{d-1}} = O(1/\sqrt{d})$.
\label{lem:beta-star}
\end{lemma}

\begin{lemma}
Suppose $d\geq 2$ and let $(\hat b_1,\hat b_2)$ be the values returned by $\textsc{BisectionSearch}(\varepsilon_s,\delta_s)$.
With probability $1-\delta_s$ the following hold: $\hat b_1\leq b^*\leq\hat b_2$,
and at most $O(\varepsilon_s^{-2}\log(1/\delta_s\varepsilon_s))=O(d\log(d/\delta))$ queried samples are consumed.
\label{lem:key-bisection-search}
\end{lemma}

Intuitively, Lemma \ref{lem:beta-star} establishes that the ``balancing'' intercept $b^*$ is $O(1/\sqrt{d})$ close to the intercept $\mu^*$ in the 
utility model, which is helpful for our later analysis.
Lemma \ref{lem:key-bisection-search} further establishes that the returned two actions $\hat b_1,\hat b_2$ sandwich the ``label-balancing'' action $b^*$,
and also upper bound the total number of labeled (queried) samples consumed in the algorithmic procedure.

\subsection{Margin-based Active Learning}\label{sec:margin-based}

In Algorithm \ref{alg:margin-based-active-learning}
we provide the pseudocode description of the margin based active learning algorithm
we use in this problem to actively learn a linear model with intercepts.

\begin{algorithm}[t]
\caption{Margin-based Active Learning Non-homogeneous Linear Classifiers}
\label{alg:margin-based-active-learning}
\begin{algorithmic}[1]
\Function{MarginBasedActiveLearning}{$b,\varepsilon_a,\delta_a,\kappa_m,\kappa_n,\epsilon_0,\beta_0$}
	\State Collect $n_0=\lceil\kappa_n/\epsilon_0^2\rceil$ samples with action $b$ and let $\mathcal D_0=\{(x,y)\}\subseteq\mathbb B_d(2)\times\{\pm 1\}$, $|\mathcal D_0|=n_0$ be the queried samples;\label{line:init-1}
	\State Let $\hat w_0,\hat\beta_0\gets\arg\min_{\|w\|_2=1,|\beta|\leq\beta_0}\sum_{(x,y)\in\mathcal D_0}\vct 1\{y\neq \sgn(\langle x,w\rangle-\beta)\}$;
	\State Let $k_0=\min\{k\in\mathbb N: 2^{-k}\epsilon_0\leq\varepsilon_a\}$;\label{line:init-2}
	\For{$k=1,2,\cdots,k_0$}
		\State $\epsilon_k\gets 2^{-k}\epsilon_0$, $m_k\gets \kappa_m\sqrt{\epsilon_k}$, $n_k\gets \lceil\kappa_n d/\epsilon_k\rceil$, $\mathcal D_k=\emptyset$;
		\While{$|\mathcal D_k|<n_k$}
			\State Observe context vector $x\in\mathbb R^d$ for the next object;
			\If{$|\langle x,\hat w_{k-1}\rangle - \hat\beta_{k-1}| \leq m_k$}
				\State Invoke action $b$ and let $y\in\{\pm 1\}$ be the collected binary feedback;
				\State Update $\mathcal D_k\gets\mathcal D_k\cup\{x,y\}$;
			\EndIf
		\EndWhile
		\State $\hat w_k,\hat\beta_k \gets \arg\min_{\|w\|_2=1,|\beta|\leq\beta_0}\sum_{(x,y)\in\mathcal D_k}\vct 1\{y\neq \sgn(\langle x,w\rangle-\beta)\}$;\label{line:01minimization}
	\EndFor
	\State \textbf{return} $\hat w_{k_0},\hat\beta_{k_0}$.
\EndFunction
\end{algorithmic}
\end{algorithm}

Note that in Algorithm \ref{alg:margin-based-active-learning}
the query point $b$ is fixed, with the algorithm only able to select which sample/contextual vector to act upon.
Since the query point $b$ is fixed, we can consider linear models with intercepts as $\hat v(\cdot)=\langle\cdot,\hat w\rangle -\hat\beta$.
For such a model, we define the \emph{error} of $\hat v$ under the query point $b$ as
\begin{equation}
\err_b(\hat v) := \Pr_{x\sim P_X,\xi\sim P_\xi}\big[\sgn(\underbrace{v(x)+\xi}_{u(x)}-b)\neq \sgn(\hat v(x))\big],
\label{eq:defn-err}
\end{equation}
where $v(x)=\langle x,w^*\rangle - \mu^*$.
Note that for any $b\in\mathbb R$, the model $v_b^*(\cdot) := \langle \cdot, w^* \rangle - \mu^* - b$
has the smallest error defined in Eq.~(\ref{eq:defn-err}), 
This is because $v_b^*(\cdot)$ is the Bayes classifier; that is, $v_b^*(x)\geq 0$ if and only if $\Pr[v(x)+\xi\geq b|x]\geq 1/2$.
Hence, we can also define the \emph{excess error} of a model $\hat v(\cdot)$ as
\begin{equation}
\Delta\err_b(\hat v) := \err_b(\hat v) - \err_b(v_b^*).
\label{eq:defn-excess-err}
\end{equation}

\begin{figure}[t]
\centering
\includegraphics[width=0.45\textwidth]{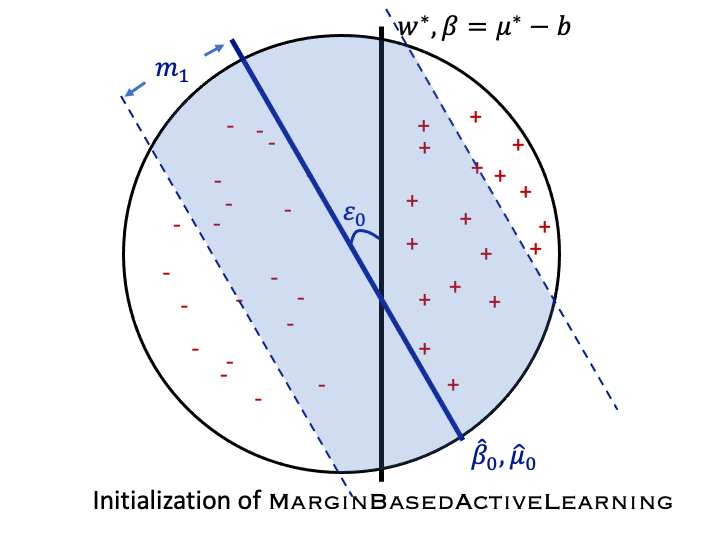}
\includegraphics[width=0.45\textwidth]{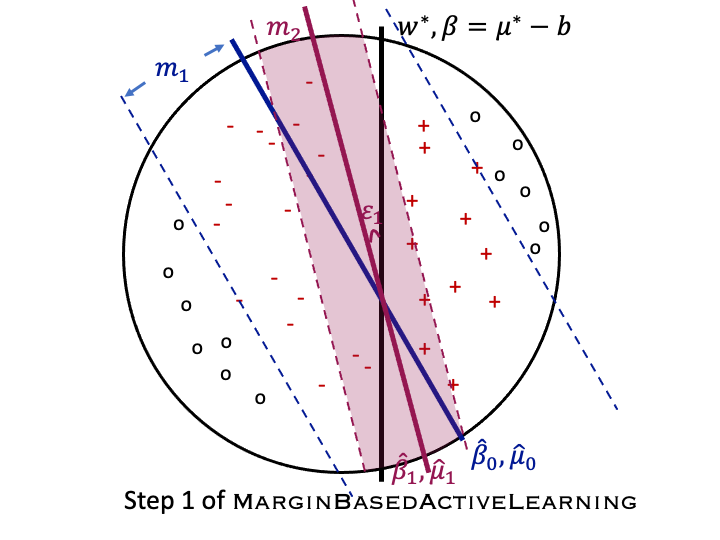}
\caption{Graphical illustration of the main idea of Algorithm \ref{alg:margin-based-active-learning}.}
{\small Note: the left panel shows the initialization step (Lines \ref{line:init-1} and \ref{line:init-2}) of Algorithm \ref{alg:margin-based-active-learning}. 
In the initialization step, sample selection is not carried out and therefore the obtained model estimates $\hat w_0,\hat \beta_0$ have error upper bounded by $\varepsilon_0$.
The right panel shows the first iteration of Algorithm \ref{alg:margin-based-active-learning}.
As shown in the figure, only those samples that are within an $m_1$ margin around $\hat w_0,\hat\beta_0$ (those within the blue dashed lines)
are labeled/queried.
After the first iteration, a more refined estimate $\hat w_1,\hat\beta_1$ is obtained and a shrunk margin $m_2$ is imposed (marked with maroon) for the next iteration.}
\label{fig:margin-based}
\end{figure}

Figure \ref{fig:margin-based} illustrates the principles of Algorithm \ref{alg:margin-based-active-learning}.
The main idea of Algorithm \ref{alg:margin-based-active-learning} is simple: 
the algorithm first uses a ``warm-up'' epoch consisting of $n_0$ queried samples to construct a preliminary model estimate $\hat w_0$ and $\hat\beta_0$.
There is no sample selection or active learning in this warm-up procedure, and the analysis of excess errors of $\hat w_0,\hat\beta_0$ follows the
standard VC theory analyzing empirical risk minimizers of binary classifiers (see e.g. Lemma \ref{lem:VC-erm} in the proof and also \cite{balcan2007margin,vapnik2015uniform,vapnik2013nature}).
Next, in each epoch the algorithm only takes action $b$ for those users with contextual vectors that are \emph{close} to the current classification hyperplane
(i.e., those users with small ``margin'' $|\langle x,\hat w_{k-1}\rangle-\hat\beta_{k-1}|$).
This concentrates our labeled/queried samples to the region that are close to the classification hyperplane, which helps reduce the number of queried
samples as the queried samples are collected on regions that are the most uncertain from a binary classification perspective.

The following lemma is the main result of this section, which is proved in Sec.~\ref{sec:proof-margin-based}.
\begin{lemma}
Let $(\hat w,\hat\beta)$ be returned by Algorithm \ref{alg:margin-based-active-learning} with parameters satisfying $|\mu^*-b|/\|w^*\|_2\leq\beta_0=O(1/\sqrt{d})$, 
$\kappa_m=\Omega(1)$, $\kappa_n=\Omega(d+\log\log(1/\varepsilon_a)+\log(1/\delta_a))$ and $\epsilon_0=\sqrt{\varepsilon_a}$.
Let $\hat v(\cdot)=\langle\cdot,\hat w\rangle-\hat\beta$.
Then for sufficiently large $d$ and sufficiently small $\varepsilon_a$, with probability $1-\delta_a$ the following hold:
\begin{enumerate}
\item $\Delta\err(\hat v)\leq \varepsilon_a$;
\item Algorithm \ref{alg:margin-based-active-learning} consumes $O(\kappa_n d/\varepsilon_a)$ queried samples and $\tilde O(\kappa_n \sqrt{d}e^{d\varepsilon_a}/\varepsilon_a^{3/2})$ total samples. 
\end{enumerate}
\label{lem:margin-based}
\end{lemma}

Essentially, Lemma \ref{lem:margin-based} shows that the estimated linear model $\hat v(\cdot)$ produced by Algorithm \ref{alg:margin-based-active-learning}
has the target excess risk $\varepsilon_a$ with high probability.
The lemma also upper bounds the number of queried and total samples consumed in the estimation procedure.
As we can see, the number of labeled samples required is on the order of $O(1/\varepsilon_a)$, which is an order of magnitude fewer
than the total number of samples consumed (on the order of $\tilde O(1/\varepsilon^{3/2})$).
This shows that the active learning procedure is capable of drastically reducing the number of queried samples required to attain an accurate model estimate $\hat v$,
by being selective in the user context vectors.

\subsection{Sample complexity analysis of Algorithm \ref{alg:meta-algorithm}}

In this section we establish the following theorem, which analyzes the sample complexity (both samples that are queried on and samples that are passed)
of Algorithm \ref{alg:meta-algorithm}, and provides guidance on the selection of the algorithm input parameters.

\begin{theorem}
Suppose Algorithm \ref{alg:meta-algorithm} is executed with $\kappa_m\asymp 1$, $\kappa_n\asymp d+\log\log(1/\varepsilon)+\log(1/\delta)$,
$\kappa_\varepsilon\asymp 1/d$ and $\beta_0\asymp 1/\sqrt{d}$. Then for sufficiently small $\varepsilon>0$ and sufficiently large $d$, with probability $1-\delta$ it holds that $|\hat v(x)-v(x)|\leq \varepsilon$ for all $x\in\mathbb B_2(d)$. Furthermore, the algorithm makes $n(\delta,\varepsilon)$ queries among $m(\delta,\varepsilon)$ total samples/contexts, with
\begin{eqnarray*}
n(\delta,\varepsilon) &=& O\left(\frac{d^3\log(d\log(d\varepsilon^{-1})/\delta)\log^2(1/\varepsilon)}{\varepsilon^2}\right), \\
m(\delta,\varepsilon) &=& O\left(\frac{d^3\log(d\log(d\varepsilon^{-1})/\delta)\log^3(1/\varepsilon)}{\varepsilon^3}\right).
\end{eqnarray*}
\label{thm:meta-algorithm}
\end{theorem}

Theorem \ref{thm:meta-algorithm} shows that, by using more unlabeled/unqueried samples than those that are labeled (more specifically, $1/\varepsilon^3$ total samples and $1/\varepsilon^2$ labeled ones), the utility function estimate $\hat v(\cdot)$ produced by our active learning algorithm is within $\varepsilon$ estimation error with high probability. In Sec.~\ref{sec:numerical} of numerical studies, we will see that the availability of unlabeled samples will greatly improve the estimation accuracy of an active learning algorithm, compared to a passive learning baseline which cannot skip or select samples to query.

\begin{remark}
If the decision-maker needs to make queries to all incoming contexts/samples  (i.e., skipping uninformative samples is not allowed), then at least $\Omega(d/\varepsilon^3)$ samples are required.
{
To see this, note that the standard classification theory establishes that $\Omega(d/\varepsilon^{3/2})$ samples are needed
to obtain a linear classifier $\hat w$ such that $\Pr[\sgn(\hat w^\top x)\neq\sgn((w^*)^\top x)]\leq\varepsilon$
(see, e.g., \cite{mammen1999smooth}, \citep[Table 1, probabilistic labels with Bayes classifier in $H$ and TNC parameter $\alpha=1/2$]{ben2014sample}).
}
On the other hand, it can be shown via an integration argument as follows. Let $\angle(\hat w, w^*)$ denote the angle between $\hat w$ and $w^*$.  If both $\hat w,  w^*$ are normalized (i.e., $\|\hat w\|_2=\|w^*\|_2=1$) and 
$\|\hat w-w^*\|_2\approx \angle(\hat w,w^*)\approx\varepsilon$ then $\Pr_x[\sgn(\hat w^\top x)\neq \sgn((w^*)^\top x)]\approx \varepsilon^2$.
This shows that in order to achieve $|\hat v(\cdot)-v(\cdot)|\leq \varepsilon$ we must have $\Pr_x[\sgn(\hat w^\top x)\neq\sgn((w^*)^\top x)]\lesssim {\varepsilon}^2$,
indicating a sample complexity lower bound of $\Omega(d/(\varepsilon^{3/2})^2) = \Omega(d/\varepsilon^3)$.

\label{rem:lower}
\end{remark}

{
\begin{remark}
When each skipped sample has a cost of $\rho\in[0,1)$ compared to a labeled sample, 
the combined sample complexity of our proposed active learning algorithm is on the order of $\widetilde O(\varepsilon^{-2}+\rho\varepsilon^{-3})$,
omitting polynomial dependency on other problem parameters.
On the other hand, an algorithm incapable of skipping samples requires $\Omega(\varepsilon^{-3})$ samples as indicated in the previous remark,
significantly higher than $\widetilde O(\varepsilon^{-2}+\rho\varepsilon^{-3})$ especially when $\rho$ is small (indicating that skipped samples 
are much less costly compared with labeled samples).
\label{rem:cost}
\end{remark}
}

\begin{remark}
In this remark we discuss an ``intermediate setting'' in which the algorithm can select action levels but not skip samples.
This intermediate setting is stronger than the passive learning setting but weaker than the active learning setting.
We remark that the intermediate setting is likely to have similar sample complexity compared with passive learning.

Consider the uniform distribution on the unit $\ell_2$ ball in $\mathbb R^d$ and let $w\in\mathbb R^d$ be an arbitrary (unknown) classifier.
	It is easy to observe that, up to polynomial constants in $d$, for every small $\epsilon>0$ the probability of $x\in\{x: \theta(x,w)\leq\epsilon\}$ is $O(\epsilon)$,
	where $\theta(\cdot,\cdot)$ denotes the angle between two vectors in $\mathbb R^d$.
	This means that, without the ability to skip samples, for a batch of $n$ samples only $O(n\epsilon)$ of them are sufficiently close to the decision boundary $w$
	to offer a good amount of information. On the other hand, active learning allows the algorithm to only collect labels/responses on the $O(n\epsilon)$ samples
	that are sufficiently close to the boundary, thus leading to more efficient usage of information from labeled samples. 	While the intermediate setting can still adaptively change the action levels (corresponding to changing the intercept in a non-homogeneous linear classification model), 
	such ability is unlikely to achieve the ``sample concentration'' effect  because only changing one parameter in a
	multi-variate linear model cannot bring a uniformly sampled data point arbitrarily close to the (unknown) decision boundary. 
\end{remark}

%
%
%

\section{Technical Proofs}
\label{sec:proof}

In this section we state the proofs of the main results in this paper.
There are also some technical lemmas that either easy to prove, or cited/rephrased from existing works,
which will be presented in the appendix.
For simplicity, let $P_U$ be the uniform distribution on $\mathbb B_2(d)=\{x\in\mathbb R^d: \|x\|_2\leq 1\}$ for all proofs in this section.

\subsection{Proof of results in Sec.~\ref{sec:alg-bisection}}\label{sec:proof-bisection}

\subsubsection{Proof of Lemma \ref{lem:beta-star}.}
First note that $\langle w^*,x\rangle-\mu^*\geq b^*$ is equivalent to $\langle w^*,x\rangle-\beta^*\geq 0$, with $\beta^* = \mu^* + b^*$.
Note also that we may assume $\|w^*\|_2=1$ because $|\beta^*|/\|w^*\|_2$ is invariant to $\|w^*\|_2$.
In this proof we shall use the lower and upper bounds of $f_x$ by connecting it with the uniform distribution on $\mathbb B_2(d)$, $P_U$.
Because $P_U$ is isotropic, we may assume without loss of generality that $w^*=(1,0,\cdots,0)$ and $\beta^*\geq 0$.
We will also abbreviate $\eta=\eta_{b^*}$ and $\Delta=\Delta_{b^*}$ since all margins in this proof are with respect to $b^*$.
Then for all $x\in\mathbb B_2(d)$ with $x_1\geq \beta^*$, $\eta(x)\geq 1/2$ and further more
$$
\eta(x) - \frac{1}{2} = \phi(x_1-\beta^*) = \int_0^{x_1-\beta^*}\phi'(u)\ud u \leq C_\xi (x_1-\beta^*). 
$$
Subsequently, by Assumption (A2) and Lemma \ref{lem:marginal-d1}, it holds that 
\begin{align}
\int_{x_1\geq\beta^*}&\left(\eta(x)-\frac{1}{2}\right) \ud P_x(x)
\leq C_x\int_{x_1\geq \beta^*}\left(\eta(x)-\frac{1}{2}\right)\ud P_U(x) \leq C_xC_\xi\int_{x_1\geq\beta^*} (x_1-\beta^*)\ud P_U(x)\nonumber\\
&\leq C_xC_\xi\int_0^1\sqrt{\frac{d+1}{2\pi}}e^{-(d-1)(\beta^*+\gamma)^2/2} \gamma\ud\gamma
\leq C_xC_\xi\sqrt{d}e^{-(d-1)(\beta^*)^2/2}\int_0^{1}\gamma e^{-(d-1)\gamma^2/2}\ud\gamma.
\label{eq:proof-beta-star-1}
\end{align}
With $\gamma\mapsto \gamma/\sqrt{d-1}$, we have
$\int_0^1\gamma e^{-(d-1)\gamma^2/2}\ud \gamma \leq \sqrt{\frac{2\pi}{d-1}}\mathbb E_{z\sim\mathcal N(0,1/(d-1))}[|z|]/2\leq 1/(d-1)$.
Noting that $\sqrt{d}\leq \sqrt{2(d-1)}$ for $d\geq 2$, Eq.~(\ref{eq:proof-beta-star-1}) can then be simplified to
\begin{equation}
\int_{x_1\geq\beta^*}\left(\eta(x)-\frac{1}{2}\right) \ud P_x(x) \leq \frac{\sqrt{2}C_xC_\xi}{\sqrt{d-1}}e^{-(d-1)(\beta^*)^2/2}.
\label{eq:proof-beta-star-2}
\end{equation}

On the other hand, for all $x\in\mathbb B_2(d)$ with $x_1\leq\beta^*$, $\eta(x)\leq 1/2$ and furthermore
$$
 \frac{1}{2}-\eta(x) = -\phi(x_1-\beta^*) = \int_0^{\beta^*-x_1} \phi'(u)\ud u \geq c_\xi(\beta^*-x_1).
$$
Subsequently, by Assumption (A2) and Lemma \ref{lem:marginal-d1}, it holds that 
\begin{align}
\int_{x_1\leq\beta^*}&\left(\frac{1}{2}-\eta(x)\right) \ud P_x(x)\geq c_x\int_{x_1\leq 0}\left(\frac{1}{2}-\eta(x)\right)\ud P_U(x)\geq c_xc_\xi\int_{x_1\leq 0}(\beta^*-x_1)\ud P_U(x)\nonumber\\
&\geq c_xc_\xi\int_{x_1\leq 0}-x_1\ud P_U(x) \geq c_xc_\xi\int_{0}^1 \sqrt{\frac{d+1}{16\pi}}e^{-(d-1)\gamma^2/2}\gamma\ud\gamma
\geq \frac{c_xc_\xi\sqrt{d}}{4\sqrt{\pi}}\int_{1/\sqrt{d-1}}^{\sqrt{2}/\sqrt{d-1}}e^{-(d-1)\gamma^2/2}\gamma\ud\gamma\nonumber\\
&\geq \frac{c_xc_\xi\sqrt{d}}{4\sqrt{\pi}}\times \frac{1}{e\sqrt{d-1}}\times \frac{\sqrt{2}-1}{\sqrt{d-1}}\geq \frac{(\sqrt{2}-1)c_xc_\xi}{4e\sqrt{\pi(d-1)}}.\label{eq:proof-beta-star-3}
\end{align}

Combining Eqs.~(\ref{eq:proof-beta-star-1},\ref{eq:proof-beta-star-3}) we obtain
$$
\frac{1}{2} = \Pr_{x\sim P_x}[y=1|b^*] \leq \frac{1}{2} + \frac{\sqrt{2}C_xC_\xi}{\sqrt{d-1}}e^{-(d-1)(\beta^*)^2/2} - \frac{(\sqrt{2}-1)c_xc_\xi}{4e\sqrt{\pi(d-1)}}.
$$
To satisfy the above inequality, $\beta^*\geq 0$ must satisfy
$$
\beta^* \leq \sqrt{\frac{2\ln(100C_xC_\xi/c_xc_\xi)}{d-1}} = O(1/\sqrt{d}),
$$
which proves Lemma \ref{lem:beta-star}. 

\subsubsection{Proof of Lemma \ref{lem:key-bisection-search}.}

For notational simplicity define $\wp(\hat b) := \Pr_{x\sim P_x}[v(x)\geq \hat b]$.
Clearly, $\wp(b^*)=1/2$ and $\wp(\cdot)$ is a monotonically decreasing function.
By Hoeffding's inequality, at sample $n$ we have $\Pr[\wp(\hat b_3)\in [\underline p_3,\overline p_3]] \geq 1-2e^{-2\hat n\times \ln(8n^2/\delta_s)/(2\hat n)}\geq 
1-\frac{\delta_s}{4n^2}$.
The same inequality holds for $\Pr[\wp(\hat b_4)\in[\underline p_4,\overline p_4]]$ as well.
By the union bound, the probability that $\wp(\hat b_3)\in[\underline p_3,\overline p_3]$ and $\wp(\hat b_4)\in[\underline p_4,\overline p_4]$
throughout the entire Algorithm \ref{alg:bisection-search} is lower bounded by
$$
1-\sum_{n\geq 1}2\times \frac{\delta_s}{4n^2} = 1-\frac{\delta_s}{2}\sum_{n\geq 1}\frac{1}{n^2}\geq 1-\frac{\delta_s}{2}\frac{\pi^2}{6}\geq 1-\delta_s.
$$
This shows that with probability $1-\delta_s$ the $(\hat b_1,\hat b_2)$ pair returned by Algorithm \ref{alg:bisection-search}
satisfies $\hat b_1\leq b^*\leq\hat b_2$ due to the monotonicity of the $\wp(\hat b)$ function.

To analyze the number of queried samples/objects by Algorithm \ref{alg:bisection-search}, we require some additional technical results.
The following lemma connects the deviation $|\wp(\hat b)-1/2|$ with $|\hat b-b^*|$.
\begin{lemma}
Recall the definition that $\wp(\hat b)=\Pr_{x\sim P_x}[w(x)\geq\hat b]$ and $b^*$ such that $\wp(b^*)=1/2$. Then
$0.07c_xc_\xi|\hat b-b^*| \leq |\wp(\hat b)-1/2|\leq C_\xi|\hat b-b^*|$.
\label{lem:beta-deviation}
\end{lemma}
\begin{proof}{Proof of Lemma \ref{lem:beta-deviation}.}
Define $\beta^* = b^*+\mu^*$ and $\hat\beta = \hat b+\mu^*$. Define also $s := \hat\beta - \beta^*$, so that $\hat\beta = \beta^*+ s$.
Recall the definition of margin that $\Delta_{b^*}(x)=v(x)-\beta^*$, and $\Pr[y=1|x,b^*]=\phi(\Delta_{b^*}(x))$.
Under $\hat b$, we have $\Delta_{\hat b}(x)=v(x)-\hat\beta = \Delta_{b^*}(x) - s$ and $\Pr[y=1|x,\hat b]=\phi(\Delta_{b^*}(x) - s)$.
Subsequently, 
\begin{align*}
\big|\wp(\hat b)-\wp(b^*)\big|
&\leq \mathbb E_{x\sim P_x}\left[\big|\phi(\Delta_{b^*}(x)-s)-\phi(\Delta_{b^*}(x))\big|\right]
\leq \sup_{|\gamma|\leq 1}\big|\phi(\gamma-s)-\phi(\gamma)\big|\leq C_\xi |s|.
\end{align*}
This proves the upper bound on $|\wp(\hat b)-1/2|$.

We next consider the lower bound of $|\wp(\hat b)-1/2|$.
Without loss of generality assume $\beta^*\geq 0$, $w^*=(1,0,\cdots,0)$ and $s\geq 0$.
We will lower bound $|\wp(\hat b)-1/2|$ by studying the decrease of $\Pr[y=1|x]$ on the ball segment $\mathbb B_2(d)\cap\{x\in\mathbb R^d: -r\leq x_1\leq 0\}$
with $r=1/\sqrt{2(d-1)}\leq 1/\sqrt{2}$ for $d\geq 2$.
More specifically, 
\begin{align}
\big|\wp(\hat b)&-1/2\big| \geq \int_{r\leq x_1\leq 0}\big[\phi(\beta^*-x_1)-\phi(\beta^*-x_1-s)\big]\ud P_X(x)\nonumber\\
&\geq c_x\int_{r\leq x_1\leq 0}\big[\phi(\beta^*-x_1)-\phi(\beta^*-x_1-s)\big]\ud P_U(x)\label{eq:proof-beta-deviation-1}\\
&\geq c_x\int_0^r \sqrt{\frac{d+1}{16\pi}}e^{-(d-1)\gamma^2/2}\big[\phi(\beta^*+\gamma)-\phi(\beta^*+\gamma-s)\big]\ud\gamma\label{eq:proof-beta-deviation-2}\\
&\geq c_x\int_0^r \sqrt{\frac{d+1}{16\pi}}e^{-(d-1)\gamma^2/2}c_\xi s\ud\gamma \label{eq:proof-beta-deviation-3}\\
&\geq c_xr \times \sqrt{\frac{d+1}{16\pi}}\times e^{-(d-1)r^2/2}\times c_\xi s = \frac{c_xc_\xi s}{4\sqrt{2\pi\sqrt{e}}} \geq 0.07c_x c_\xi s.
\end{align}
Here Eq.~(\ref{eq:proof-beta-deviation-1}) is due to Assumption (A2), Eq.~(\ref{eq:proof-beta-deviation-2}) is due to Lemma \ref{lem:marginal-d1},
and Eq.~(\ref{eq:proof-beta-deviation-3}) is due to Assumption (A3). $\square$
\end{proof}

We are now ready to analyze the number of queried samples in Algorithm \ref{alg:bisection-search}.
Fix an arbitrary pair of $(\hat b_1,\hat b_2)$ at outer iteration $\tau$ such that $\hat b_2-\hat b_1=\varepsilon_\tau = 2(2/3)^\tau B\geq\varepsilon_a$.
Then either $|\hat b_3-b^*|\geq \varepsilon_\tau/6$ or $|\hat b_4-b^*|\geq\varepsilon_\tau/6$.
Let $\hat n_\tau$ be the final count when outer iteration $\tau$ ends.
The condition $0.5\in[\underline p_3,\overline p_3]\wedge 0.5\in[\underline p_4,\overline p_4]$ in the inside while loop will be violated if
$\sqrt{\ln(8n_\tau^2/\delta_s)/2\hat n_\tau} < 0.07c_xc_\xi\varepsilon_\tau/6$, which translates to
$$
\hat n_\tau \leq 1 + \frac{7500\ln(8n_\tau/\delta_s)}{c_x^2c_\xi^2\varepsilon_\tau^2},
$$ \
where $n_\tau=\sum_{\tau'\leq\tau}\hat n_{\tau'}$. Let $\tau_0$ be the largest integer such that $\varepsilon_{\tau_0}\geq\varepsilon_a$.
Then the total number of queried samples is upper bounded by
$$
2\sum_{\tau\leq\tau_0}\hat n_\tau = O\left(\sum_{\tau\leq\tau_0}\frac{1}{\varepsilon_\tau^2}\log\left(\frac{1}{\delta_s\varepsilon_\tau}\right)\right)
= O(\varepsilon_{\tau_0}^{-2}\log(1/(\delta_s\varepsilon_{\tau_0}))) = O(\varepsilon_a^{-2}\log(1/(\delta_s\varepsilon_a))).
$$

\subsection{Proof of results in Sec.~\ref{sec:margin-based}}\label{sec:proof-margin-based}

The objective of this section is to prove the key Lemma \ref{lem:margin-based}.
Throughout this proof we assume that $d$ is sufficiently large and $\varepsilon_a>0$ is sufficiently small.
We also define $\theta(w,w')$ as the smallest angle between $w,w'\in\mathbb R^d$.

Recall the definition that $v^*_b(\cdot) = \langle\cdot,w^*\rangle -\beta$ with $\beta=b+\mu^*$ is the non-homogeneous linear classifier
with the smallest classification error.
For presentation simplicity, we shall normalize $v^*_b(\cdot)$ (since only the \emph{signs} of $v^*_b(\cdot)$ matter in a binary classification problem)
as $v^*_b(\cdot)=\langle\cdot,\tilde w^*\rangle - \tilde\beta$ where $\tilde w^*=w^*/\|w^*\|_2$ and $\tilde\beta=\beta/\|w^*\|_2$.
Our first technical lemma shows that if another classifier $\hat v(\cdot)=\langle\cdot,\hat w\rangle - \hat\beta$ has small excess error,
then the angle between $\hat w$ and $\tilde w^*$ must be small.
\begin{lemma}
Let $\hat v(\cdot)=\langle\cdot,\hat w\rangle-\hat\beta$, $\|\hat w\|_2=1$ be a learnt classifier
such that $\Delta\err(\hat v)=\err(\hat v)-\err(v^*_b) \leq \epsilon$. Then for sufficiently small $\epsilon$, it holds that
$\tan\theta(\hat w,\tilde w^*)\leq 23e^{(d-2)\beta_0^2/2}\sqrt{\epsilon} = O(\sqrt{\epsilon})$.
\label{lem:error-to-angle}
\end{lemma}
\begin{proof}{Proof of Lemma \ref{lem:error-to-angle}.}
Abbreviate $\theta = \theta(\hat w,\tilde w^*)$. Without loss of generality, assume $w^*=(1,0,\cdots,0)$, $\hat w=(1-\cos\theta,\sin\theta,0,\cdots,0)$ and $\tilde\beta\geq 0$.
For sufficiently small $\epsilon$, we have $\tan\theta\leq 1$, and {the disagreement region between $\hat v$ and $v_b^*$ is depicted in blue in the left panel of Figure \ref{fig:illustration-error-to-angle}}.
Note also that, as one adjusts the intercept $\hat\beta$ in $\hat v$, one disagreement region will enlarge and the other one will shrink.
{As a result, the minimal disagreement region is depicted in yellow in the middle panel of Figure \ref{fig:illustration-error-to-angle},}
with the radius $\rho$ to be at least $1/2$ for sufficiently large $d$ since $|\beta|\leq\beta_0=O(1/\sqrt{d})$.
To further simplify, we take only the upper triangle of the disagreement region with $r=1/2\sqrt{d}\leq \rho$ and study the rectangular region {designated as $\Omega$ in the right panel of Figure \ref{fig:illustration-error-to-angle}}, whose size is $\frac{r}{2}\times h$ where $h=\frac{r}{2}\tan\theta$.

\begin{figure}[t]
\centering
\includegraphics[width=0.32\textwidth]{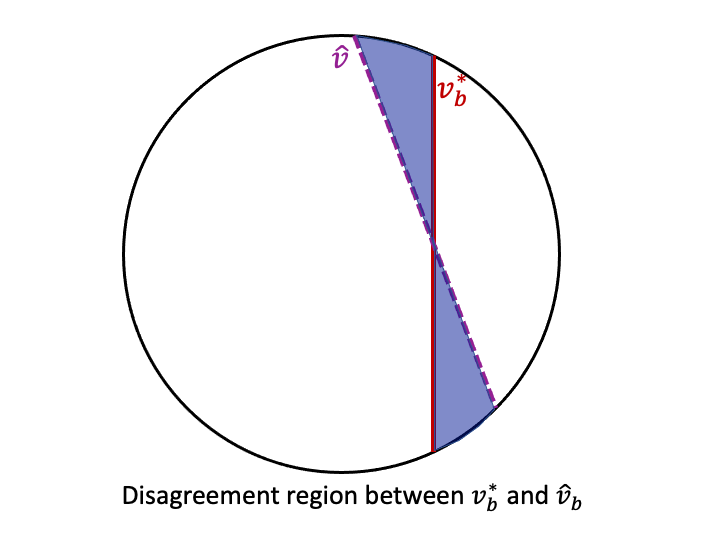}
\includegraphics[width=0.32\textwidth]{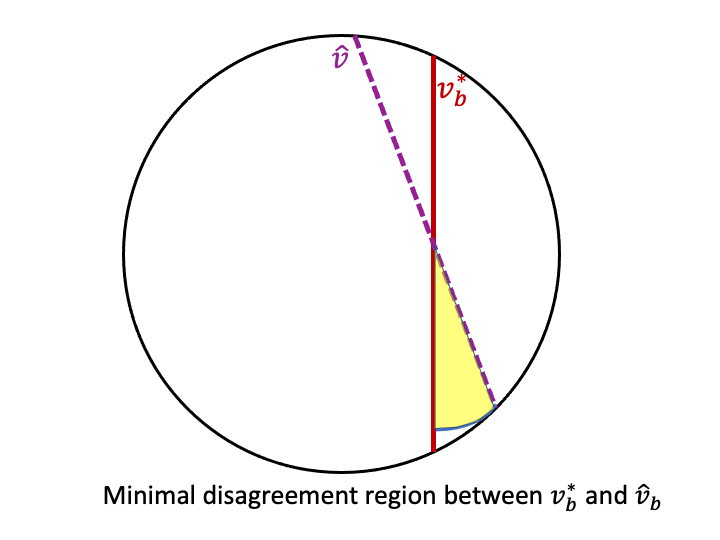}
\includegraphics[width=0.32\textwidth]{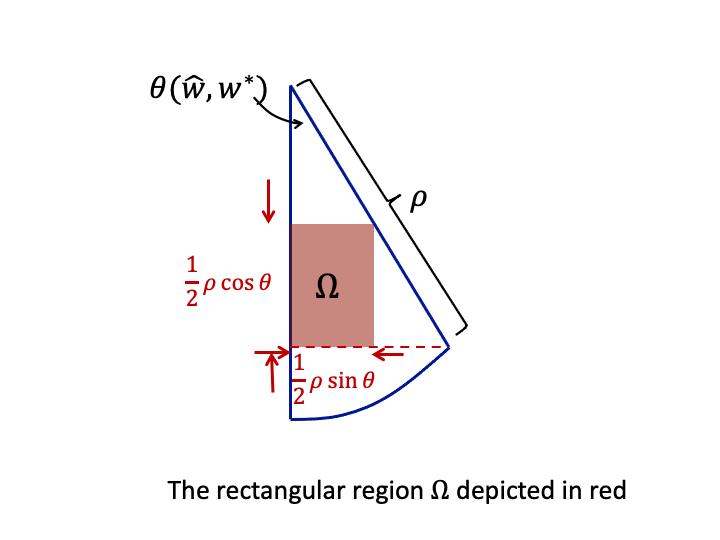}
\caption{Graphical illustration of the proof of Lemma \ref{lem:error-to-angle}.}
\label{fig:illustration-error-to-angle}
\end{figure}

The excess error of $\hat v$ can be lower bounded by the deviation of $\eta(x)=\Pr[y=1|x,b]$ from $1/2$ on $\Omega$. 
More specifically, 
\begin{align}
\Delta\err(\hat v)& \geq \int_{\Omega}\left(\eta(x)-\frac{1}{2}\right)\ud P_X(x)
\geq c_x\int_{\Omega}\left(\eta(x)-\frac{1}{2}\right)\ud P_U(x) = c_x\int_{\Omega}\phi(x_1-\beta)\ud P_U(x)\label{eq:proof-error-to-angle-1}\\
&\geq c_xc_\xi\int_\Omega (x_1-\beta)\ud P_U(x)\geq c_xc_\xi\int_\Omega \frac{d}{4\pi}e^{-(d-2)(x_1^2+x_2^2)/2}(x_1-\beta)\ud x_1\ud x_2\label{eq:proof-error-to-angle-2}\\
&\geq \frac{c_xc_\xi d}{4\pi}\int_0^h \gamma e^{-(d-2)((\beta+h)^2+r^2)/2}\ud\gamma\geq \frac{c_xc_\xi d}{4\pi}e^{-(d-2)(\beta^2+r^2)} \frac{1}{2}h^2
\geq \frac{e^{-(d-2)\beta^2}}{128\pi\sqrt[4]{e}}\tan^2\theta.\label{eq:proof-error-to-angle-3}
\end{align}
Here, Eq.~(\ref{eq:proof-error-to-angle-1}) is due to Assumption (A2) and the definition of $\phi$;
Eq.~(\ref{eq:proof-error-to-angle-2}) is due to Assumption (A3) and Lemma \ref{lem:marginal-d2}.
Taking the square root on both sides of Eq.~(\ref{eq:proof-error-to-angle-3}) and noting that $\Delta\err(\hat v)=\epsilon$, we complete
the proof of Lemma \ref{lem:error-to-angle}. $\square$
\end{proof}

The next lemma shows that if $\Delta\err(\hat v)$ is small, then the intercept $\hat\beta$ cannot be too far away from $\beta$ either.
\begin{lemma}
Let $\hat v(\cdot)=\langle\cdot,\hat w\rangle -\hat\beta$, $\|\hat w\|_2=1$, $|\hat\beta|\leq\beta_0$ be a learnt classifier such that
$\Delta\err(\hat v)=\err(\hat v)-\err(v_b^*)=\epsilon$. 
Then for sufficiently small $\epsilon$, $|\hat\beta-\tilde\beta|\leq 3601\sqrt{\pi}C_xC_\xi c_x^{-1}c_\xi^{-1}\max\{e^{(d-1)\beta_0^2}, 1\} \sqrt{\epsilon}$
$= O(\sqrt{\epsilon})$.
\label{lem:error-to-intercept}
\end{lemma}
\begin{proof}{Proof of Lemma \ref{lem:error-to-intercept}.}
Let $\theta=\theta(\hat w,\tilde w^*)$, and assume without loss of generality that $\tilde w^*=(1,0,\cdots,0)$ and $\tilde\beta\geq 0$.
Let also $\Delta_\beta = \hat\beta-\tilde\beta$.

\begin{figure}[t]
\centering
\includegraphics[width=0.32\textwidth]{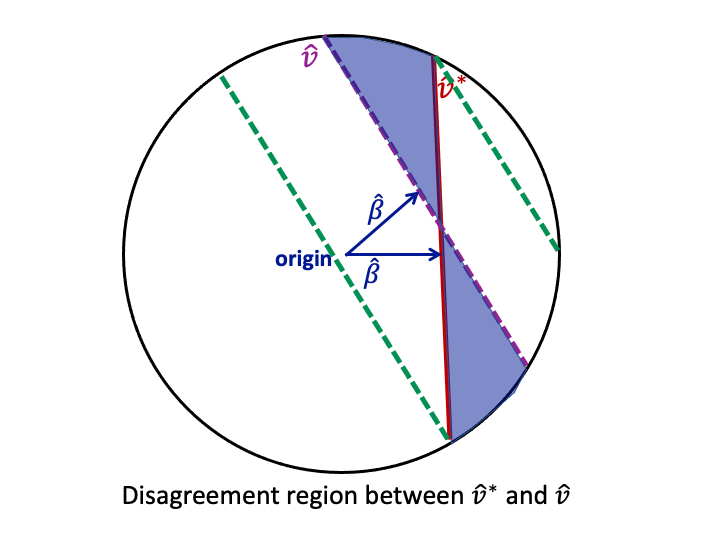}
\includegraphics[width=0.32\textwidth]{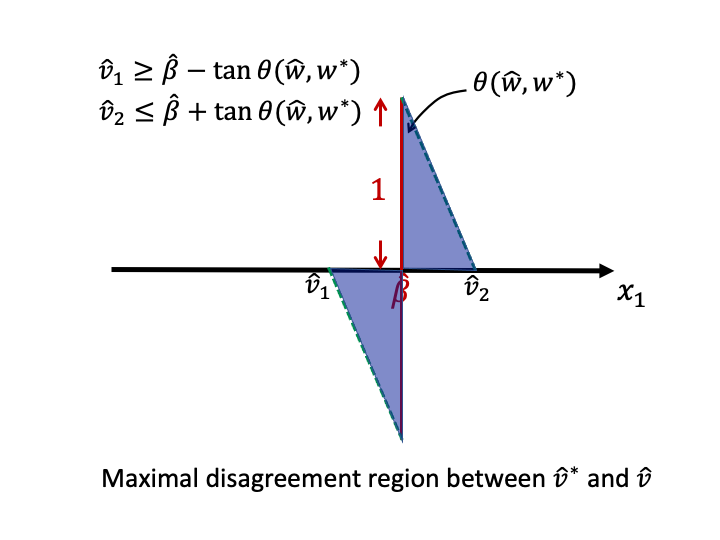}
\includegraphics[width=0.32\textwidth]{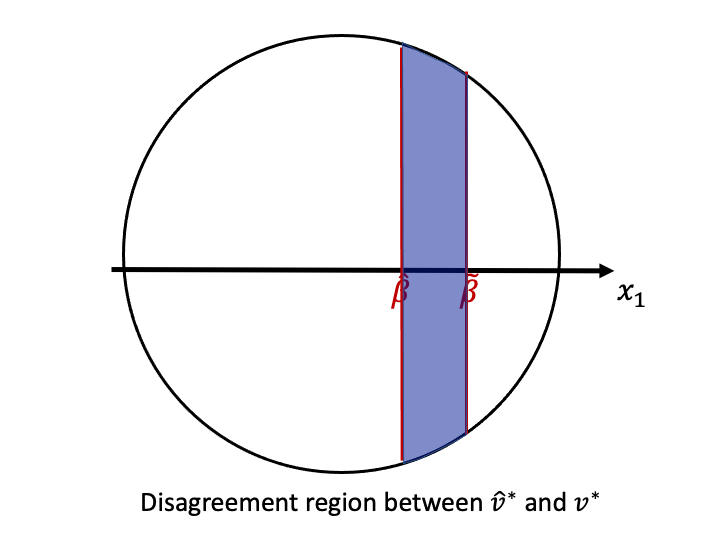}
\caption{Graphical illustration of proof of Lemma \ref{lem:error-to-intercept}.}
\label{fig:illustration-error-to-intercept}
\end{figure}

First we compare the two models of $\hat v(\cdot)=\langle\cdot,\hat w\rangle-\hat\beta$ and $\hat v^*(\cdot)=\langle\cdot,\tilde w^*\rangle - \hat\beta$.
When $d$ is sufficiently large and $\epsilon$ is sufficiently small, {the disagreement region between $\hat v$ and $\hat v^*$
is depicted in the left panel of Figure \ref{fig:illustration-error-to-intercept}.}
Since the two line segments intersect when $|\hat\beta|\leq\beta_0\to 0$ as $d\to\infty$ and $\theta(\hat w,\tilde w^*)\to 0$ as $\epsilon\to 0$,
{the maximal disagreement region between $\hat v$ and $\hat v^*$ is reached by the two green dashed lines in the left panel of 
Figure \ref{fig:illustration-error-to-intercept}, an upper bound of which is depicted in the middle panel of Figure \ref{fig:illustration-error-to-intercept}
by projecting onto the 1-dimensional space along the direction of $\tilde w^*$.}
Subsequently, the disagreement between $\hat v$ and $\hat v^*$ can be upper bounded by
\begin{align}
&\int_{\sgn(\hat v(x))\neq \sgn(\hat v^*(x))}\left|\eta(x)-\frac{1}{2}\right|\ud P_X(x) \leq \int_{x_1\in[\hat\beta\pm \tan\theta]}\left|\eta(x)-\frac{1}{2}\right|\ud P_X(x)
\leq C_xC_\xi\int_{x_1\in[\hat\beta\pm\tan\theta]}\big|x_1-\tilde\beta\big|\ud P_U(x)\nonumber\\
&\leq C_xC_\xi\int_{-\tan\theta}^{\tan\theta}\sqrt{\frac{d+1}{4\pi}}e^{-(d-1)(\hat\beta+\gamma)^2/2}|\gamma-\Delta_\beta|\ud\gamma\label{eq:proof-error-to-intercept-1}\\
&\leq C_xC_\xi\sqrt{\frac{d+1}{4\pi}}\int_{-\tan\theta}^{\tan\theta}|\gamma-\Delta_\beta|\ud\gamma= C_xC_\xi\sqrt{\frac{d+1}{4\pi}}
\left(\tan^2\theta+2\Delta_\beta\tan\theta\right) \nonumber\\
&\leq 150C_xC_\xi\sqrt{d+1}\left(e^{(d-2)\beta_0^2}\epsilon + 2e^{(d-2)\beta_0^2/2}\sqrt{\epsilon}\Delta_\beta\right)\label{eq:proof-error-to-intercept-2}\\
&= O(\sqrt{d}\epsilon+\sqrt{d\epsilon}\Delta_\beta).\nonumber
\end{align}
Here, Eq.~(\ref{eq:proof-error-to-intercept-1}) is due to Lemma \ref{lem:marginal-d1},  
and the last inequality of Eq.~(\ref{eq:proof-error-to-intercept-2}) holds by Lemma \ref{lem:error-to-angle}.

Next, consider the disagreement between the two models of $\hat v^*(\cdot)=\langle\cdot,\tilde w^*\rangle-\hat\beta$
and $v^*(\cdot)=\langle\cdot,\tilde w^*\rangle - \tilde\beta$.
First consider the case of $\hat\beta\geq \tilde\beta$, and let $\Delta_\beta = \hat\beta-\tilde\beta$.
{The disagreement region between $\hat v^*$ and $v^*$ in this case is depicted in the right panel of Figure \ref{fig:illustration-error-to-intercept}.}
The disagreement between $\hat v^*$ and $v^*$ can then be lower bounded by 
\begin{align}
&\int_{\sgn(\hat v^*(x))\neq\sgn(v^*(x))}\left(\eta(x)-\frac{1}{2}\right)\ud P_X(x)\geq \int_{x_1\in[\tilde\beta,\hat\beta]}\left(\eta(x)-\frac{1}{2}\right)\ud P_X(x)
\geq c_xc_\xi\int_{x_1\in[\tilde\beta,\hat\beta]}(x_1-\tilde\beta)\ud P_U(x)\nonumber\\
&\geq c_xc_\xi\int_{0}^{\Delta_\beta}\sqrt{\frac{d+1}{16\pi}}e^{-(d-1)(\tilde\beta+\gamma)^2/2}\gamma\ud\gamma
\geq c_xc_\xi\sqrt{\frac{d+1}{16\pi}}e^{-(d-1)\beta_0^2/2}\int_0^{\Delta_\beta}\gamma\ud\gamma\label{eq:proof-error-to-intercept-2}\\
&\geq \frac{c_xc_\xi\sqrt{d+1}}{8\sqrt{\pi}}e^{-(d-1)\beta_0^2/2}\Delta_\beta^2.\label{eq:proof-error-to-intercept-3}
\end{align}
Here the second inequality in Eq.~(\ref{eq:proof-error-to-intercept-2}) holds because $\hat\beta=\beta+\Delta_\beta\leq\beta_0$ by optimization
constraint.
If $\hat\beta<\tilde\beta$, the disagreement region has more density because the region (the $[\hat\beta,\tilde\beta]$ strip) is closer to the origin
than the perimeter of the $\mathbb B_2(d)$ ball.

Combining Eqs.~(\ref{eq:proof-error-to-intercept-2},\ref{eq:proof-error-to-intercept-3}) and noting that $v^*$ is the Bayes classifier (i.e., the classifier
that minimizes classification error), we have that
\begin{align}
\Delta\err(\hat v) &\geq \frac{c_xc_\xi\sqrt{d+1}}{8\sqrt{\pi}}e^{-(d-1)\beta_0^2/2}\Delta_\beta^2 - 150C_xC_\xi\sqrt{d+1}\left(e^{(d-2)\beta_0^2}\epsilon + 2e^{(d-2)\beta_0^2/2}\sqrt{\epsilon}\Delta_\beta\right).\label{eq:proof-error-to-intercept-4}
\end{align}
Since $\Delta\err(\hat v)=\epsilon$ and $|\Delta_\beta|\leq 2\beta_0=O(1/\sqrt{d})$, for sufficiently large $d$ the above inequality solves to
$$
|\Delta_\beta|\leq 3601\sqrt{\pi}\frac{C_xC_\xi}{c_xc_\xi}\max\big\{e^{(d-1)\beta_0^2}, 1\big\}\times \sqrt{\epsilon} = O(\sqrt{\epsilon}),
$$
which is to be proved. $\square$
\end{proof}

We are now ready to prove the key Lemma \ref{lem:margin-based} in Sec.~\ref{sec:margin-based}.
\begin{proof}{Proof of Lemma \ref{lem:margin-based}.}
Recall the definition that $\epsilon_k=2^{-k}\epsilon_0$.
We use mathematical induction to prove that, at the end of each outer iteration $k\in\{0,1,2,\cdots,k_0\}$, with probability $1-\delta_a/(k_0+1)$
it holds that $\Delta\err(\hat w_k,\hat\beta_k)\leq\epsilon_k$.

\paragraph{Base of induction.} For $k=0$, invoke Lemma \ref{lem:VC-erm} with $n=n_0$ and $\delta=\delta_a/(k_0+1)$, we have with probability $1-\delta$
that $\Delta\err(\hat w_0,\hat\mu_0)\leq O(\sqrt{\frac{d+\ln(k_0/\delta_a)}{n_0}})$.
Note also that $k_0\asymp \log(1/\varepsilon_a)$.
Hence, with $n_0=\Omega(\varepsilon_a^{-1}(d+\log\log(1/\varepsilon_a)+\log(1/\delta_a))$, we have with probability $1-\delta_a/(k_0+1)$ that
$\Delta\err(\hat w_0,\hat\mu_0)\leq\epsilon_0$.

\paragraph{Inductive steps.} We assume the inductive hypothesis is true for $k-1$, i.e., $\Delta\err(\hat w_{k-1},\hat\mu_{k-1})\leq \epsilon_{k-1}=2^{-(k-1)}\epsilon_0$. We will prove in this step that $\Delta\err(\hat w_k,\hat\mu_k)\leq\epsilon_k=2^{-k}\epsilon_0$ with probability $1-\delta_a/(k_0+1)$.

Denote $S_1=\{x\in\mathbb B_2(d): |\langle x,\hat w_{k-1}\rangle - \hat\beta_{k-1}|\leq m_k\}$ and $S_2 = \mathbb B_2(d)\backslash S_1$.
Because $\Delta\err(\hat w_{k-1},\hat\beta_{k-1})\leq\epsilon_{k-1}$, by Lemmas \ref{lem:error-to-angle} and \ref{lem:error-to-intercept}
we have that $\tan\theta(\hat w_{k-1},w^*)\leq C\sqrt{\varepsilon_{k-1}}$ and $|\hat\beta_{k-1}-\beta|\leq C\sqrt{\varepsilon_{k-1}}$ for some constant $C$
depending only on $C_x,c_x,C_\xi,c_\xi$.
Hence, with $m_k$ selected as $m_k=\Omega(\sqrt{\epsilon_{k-1}})$, 
we have that $\sgn(x^\top\hat w_{k-1}-\hat\beta_{k-1})=\sgn(x^\top w^*-\beta)$ for all $x\in S_2$.
Subsequently,
\begin{equation}
\Delta\err(\hat w_k,\hat\beta_k) = [\err(\hat w_k,\hat\beta_k|S_1) - \err(w^*,\beta|S_1)]\Pr[x\in S_1] =: \Delta\err(\hat w_k,\hat\beta_k|S_1)\Pr[x\in S_1], 
\label{eq:proof-margin-based-1}
\end{equation}
where $\err(w,\beta|S_1) = \Pr_{(x,y)}[y\neq \sgn(w^\top x-\beta)|x\in S_1]$.

Invoking Lemma \ref{lem:VC-erm}, if $n_k\geq \Omega(C_x^2m_k^2d^2\epsilon_k^{-2}\ln(k_0/\delta_a))=\Omega(d^2\epsilon_k^{-1}\ln(k_0/\delta_a))$ then it holds with probability $1-\delta_a/(k_0+1)$ that 
\begin{equation}
\Delta\err(\hat w_k,\hat\beta_k|S_1)\leq \frac{\epsilon_k}{1.3C_xm_k\sqrt{d}}.
\label{eq:proof-margin-based-2}
\end{equation}
On the other hand, we have that
\begin{align}
\Pr_{x\sim P_X}[x\in S_1] &\leq C_x\Pr_{x\sim P_U}[x\in S_1]\leq C_x \Pr_{x\sim P_U}[|x_1|\leq m_k]\leq C_x\sqrt{\frac{2(d+1)}{\pi}}\int_0^{m_k}e^{-(d-1)u^2/2}\ud u\label{eq:proof-margin-based-3}\\
&\leq 1.3C_xm_k\sqrt{d}.\label{eq:proof-margin-based-4}
\end{align}
Here, the last inequality in Eq.~(\ref{eq:proof-margin-based-3}) holds by invoking Lemma \ref{lem:marginal-d1}.
Plug Eqs.~(\ref{eq:proof-margin-based-2},\ref{eq:proof-margin-based-4}) into Eq.~(\ref{eq:proof-margin-based-1}).
We proved that $\Delta\err(\hat w_k,\hat\beta_k)\leq\epsilon_k$, which completes the induction step. 

In the final part of the proof we upper bound the total number of labeled (queried) and unlabeled samples used in Algorithm \ref{alg:margin-based-active-learning}.
The number of labeled samples is simply $n_0 + \sum_{k=1}^{k_0}n_k$. It can be upper bounded by 
\begin{align*}
n_0 + \sum_{k=1}^{k_0}n_k &\leq O(\kappa_n \epsilon_0^{-2}) + \sum_{k=1}^{k_0}O(\kappa_nd\epsilon_k^{-1})
\leq O(\kappa_n)\times \left(\frac{1}{\varepsilon_a} + \sum_{k=1}^{k_0}\frac{2^k d}{\sqrt{\varepsilon_a}}\right)
\leq O\left(\frac{\kappa_n d}{\varepsilon_a}\right),
\end{align*}
where the last inequality holds because $k_0=\min\{k\in\mathbb M: 2^{-k}\epsilon_0\leq\varepsilon_a\}$ and $\epsilon_0=\sqrt{\varepsilon_a}$.
This shows that the total number of labeled samples consumed is on the order of $O(\kappa_n d/\varepsilon_a)$.

To upper bound the total number of samples (labeled/queried or unlabeled/not queried), note that at epoch $k$ the number of total samples
is upper bounded by $\tilde O(n_k/\Pr[x\in S_1(k)])$, where $S_1(k)=\{x\in\mathbb B_2(d): |x^\top\hat w_{k-1}-\hat\beta_{k-1}|\leq m_k\}$.
Because $|\hat\beta_{k-1}|\leq \beta_0$, we can lower bound $\Pr[x\in S_1(k)]$ as
\begin{align*}
\Pr_{x\sim P_X}[x\in S_1(k)] &\geq c_x\Pr_{x\sim P_U}[x\in S_1(k)] \geq c_x\Pr_{x\sim P_U}[|x-\hat\beta_{k-1}|\leq m_k]\\
&\geq c_x\sqrt{\frac{d+1}{16\pi}}\int_0^{m_k}e^{-(d-1)(|\hat\beta_{k-1}|+u)^2/2}\ud u\\
&\geq c_x\sqrt{\frac{d+1}{16\pi}} e^{-(d-1)\beta_0^2}\times m_k e^{-(d-1)m_k^2}\\
&\geq \Omega(\sqrt{d})\times m_ke^{-(d-1)m_k^2}.
\end{align*}
Hence, the total number of samples consumed can be upper bounded by
\begin{align*}
n_0+\sum_{k=1}^{k_0}\tilde O\left(\frac{n_k}{\Pr[x\in S_1(k)]}\right) 
&\leq \tilde O(\kappa_n\epsilon_0^{-2}) + \sum_{k=1}^{k_0}\tilde O\left(\frac{\kappa_n \sqrt{d} e^{(d-1)m_k^2}}{m_k\epsilon_k}\right)\\
&\leq \tilde O(\kappa_n\epsilon_0^{-2}) + \sum_{k=1}^{k_0}\tilde O\left(\frac{\kappa_n \sqrt{d} e^{(d-1)\times \tilde O(\epsilon_k)}}{\epsilon_k^{3/2}}\right)\\
&\leq \tilde O\left(\frac{\kappa_n \sqrt{d} e^{d\varepsilon_a}}{\varepsilon_a^{3/2}}\right).
\end{align*}
This completes the proof of Lemma \ref{lem:margin-based}. $\square$
\end{proof}

\subsection{Proof of Theorem \ref{thm:meta-algorithm}}

Recall the definition that $v(\cdot)=\langle\cdot,w^*\rangle - \mu^*$.
Define $\alpha := \|w^*\|_2\leq B$, $\tilde w^*=w^*/\alpha$, and for $j\in\{1,2\}$ define $\tilde\beta_j = (\mu^*+\hat b_j)/\alpha$.
By Lemma \ref{lem:margin-based}, we have $\Delta\err(\hat w_j,\hat\beta_j)\leq \varepsilon_a$,
which by Lemmas \ref{lem:error-to-angle} and \ref{lem:error-to-intercept} implies $\tan\theta(\hat w_j,\tilde w^*)=O(\sqrt{\varepsilon_a})$
and $|\hat\beta_j-\tilde\beta_j|=O(\sqrt{\varepsilon_a})$.
This implies that $|\alpha\hat\beta_j-\mu^*-\hat b_j|\leq \alpha\times O(\sqrt{\varepsilon_a})= O(B\sqrt{\varepsilon_a})$.
On the other hand, the stopping condition in Algorithm \ref{alg:margin-based-active-learning} implies $|\hat b_2-\hat b_1|=\Omega(\varepsilon_s)=\Omega(1/\sqrt{d})$, which yields $|\hat\beta_2-\hat\beta_1|=\Omega(1/(\alpha\sqrt{d}))$ for sufficiently small $\varepsilon$ because $\hat\beta_1\to (\mu^*+\hat b_1)/\alpha$
and $\hat\beta_2\to(\mu^*+\hat b_2)/\alpha$ as $\varepsilon\to 0$. Subsequently, 
\begin{align}
\big|\hat\alpha-\alpha\big|
&= \left|\alpha - \frac{\alpha\hat\beta_2-\mu^*\pm O(B\sqrt{\varepsilon_a}) - \alpha\hat\beta_1+\mu^*\pm O(B\sqrt{\varepsilon_a})}{\hat\beta_2-\hat\beta_1}\right|\nonumber\\
&= \frac{O(B\sqrt{\varepsilon_a})}{|\hat\beta_2-\hat\beta_1|} = O(B^2\sqrt{d\varepsilon_a}).\label{eq:proof-meta-algorithm-1}
\end{align}

We now upper bound $|\hat\mu-\mu^*|$ and $\|\hat w-w^*\|_2$. By definition, $\hat\mu=\hat\alpha\hat\beta_1-\hat b_1$ and $\mu^*=\alpha\tilde\beta_1-\hat b_1$. Subsequently, 
\begin{align}
\big|\hat\mu-\mu^*\big| &\leq \big|\hat\alpha-\alpha\big|\cdot|\hat\beta_1| + \alpha\big|\hat\beta_1-\tilde\beta_1\big|\leq O(B^2\sqrt{d\varepsilon_a}\beta_0) + O(B\sqrt{\varepsilon_a})\leq O(B^2\sqrt{d\varepsilon_a}).
\label{eq:proof-meta-algorithm-2}
\end{align}
Similarly, $\hat w=\hat\alpha\hat w_1$ and $w^*=\alpha \tilde w^*$. Therefore,
\begin{align}
\|\hat w-w^*\| &\leq |\hat\alpha-\alpha|\cdot \|\hat w\|_2 + \alpha\|\hat w_1-\tilde w^*\|_2 \leq O(B^2\sqrt{d\varepsilon_a}) + O(B\sqrt{\varepsilon_a}) = O(B^2\sqrt{d\varepsilon_a}).
\label{eq:proof-meta-algorithm-3}
\end{align}
With the choice of $\varepsilon_a = \kappa_\varepsilon \varepsilon^2/\ln^2(1/\varepsilon)$ and $\kappa_\varepsilon\asymp 1/d$,
and with $\varepsilon\to 0$ being sufficiently small, Eqs.~(\ref{eq:proof-meta-algorithm-2},\ref{eq:proof-meta-algorithm-3}) yield that $\sup_{x\in\mathbb B_2(d)}|\hat v(x)-v^*(x)|\leq \varepsilon$.
Finally, plugging in the expression of $\varepsilon_a=\kappa_\varepsilon\varepsilon^2/\ln^2(1/\varepsilon)$ and invoking Lemmas \ref{lem:key-bisection-search}, \ref{lem:margin-based} we obtain the upper bounds on $n(\varepsilon,\delta)$ and $m(\varepsilon,\delta)$.

\section{Numerical results}
\label{sec:numerical}

We use synthetic data to study the numerical performance of our proposed active learning methods and compare it with baseline methods.
The main baseline method we are comparing against is a passive learning method:
\begin{itemize}
\item \textbf{The baseline method} will first invoke the \textsc{TrisectionSearch} routine in Algorithm \ref{alg:bisection-search} to obtain actions $\hat b_1,\hat b_2$.
The method then divides the remaining number of samples into two halves and use Logistic regression to form two model estimates $\hat w_1,\hat\beta_1$
and $\hat w_2,\hat\beta_2$ under actions $\hat b_1$ and $\hat b_2$ respectively, \emph{without} sample selection.
The method finally uses Lines \ref{line:defn-hat-alpha} and \ref{line:defn-hat-v} of Algorithm \ref{alg:bisection-search} to produce an estimate $\hat v(\cdot)$ of the utility function $v(\cdot)$. 
\end{itemize}

Note that, theoretically, a passive learning baseline algorithm cannot adaptively change actions in queries. However, 
we observe in our simulations that if the default actions for passive learning are too far away from optimal, very little information is gained
and the accuracy of passive learning is very low. Therefore, we use the actions estimated by the \textsc{TrisectionSearch} routine
as the default actions of a passive learning algorithm in our experiments to form a more reasonable comparison.
 
We also mention details of the implementation of our proposed active learning algorithm.
The implementation slightly deviates from the descriptions of the algorithms and the selection of parameter values in the theoretical results,
due to computational efficiency issues and other factors we observe could impact the algorithm's numerical performances.
In Line \ref{line:01minimization} the 0/1-error empirical risk minimization step is replaced with Logistic regression as the former formulation is computationally expensive. We also remove the $\|w\|_2=1,|\beta|\leq\beta_0$ constraints in the optimization but normalize the estimator after optimization. The parameters of Algorithm \ref{alg:bisection-search} are set as $\varepsilon_s=0.5$ and $\delta_s=0.1$. The parameters of Algorithm \ref{alg:margin-based-active-learning} are set as $\varepsilon_0=0.2$, $\kappa_m=1.0$ and $\kappa_n=d+\ln(n)$.
Note that we no longer need the $\beta_0$ parameter with the Logistic regression formulation.

For the problem settings, we adopt $P_X=P_U$ being the uniform distribution on the $d$-dimensional $\ell_2$ ball $\mathbb B_2(d)$.
We set the mean utility model $v(\cdot)$ as $v(\cdot)=\langle\cdot, \theta^*\rangle-\mu^*$ with 
$\theta^*=(2/\sqrt{d},\cdots,2/\sqrt{d})$, and $\mu^*=-2.5$. The noise distribution $P_\xi$ is set as the uniform distribution on interval $[-1,1]$.

\subsection{Convergence of utility estimates}

In the first set of reports we report how fast the utility estimates $\hat v(\cdot)$ of our proposed algorithm (and the passive learning baseline)
converge to the ground truth $v(\cdot)$ as the number of labeled (queried) samples $n$ increases.
The estimation errors between $\hat v(\cdot)=\langle\cdot,\hat w\rangle-\hat\mu$ and $v(\cdot)=\langle\cdot,w^*\rangle-\mu^*$
are reported as $\|\hat w-w^*\|_2+|\hat \mu-\mu^*|$.

\begin{figure}[t]
\centering
\includegraphics[width=0.32\textwidth]{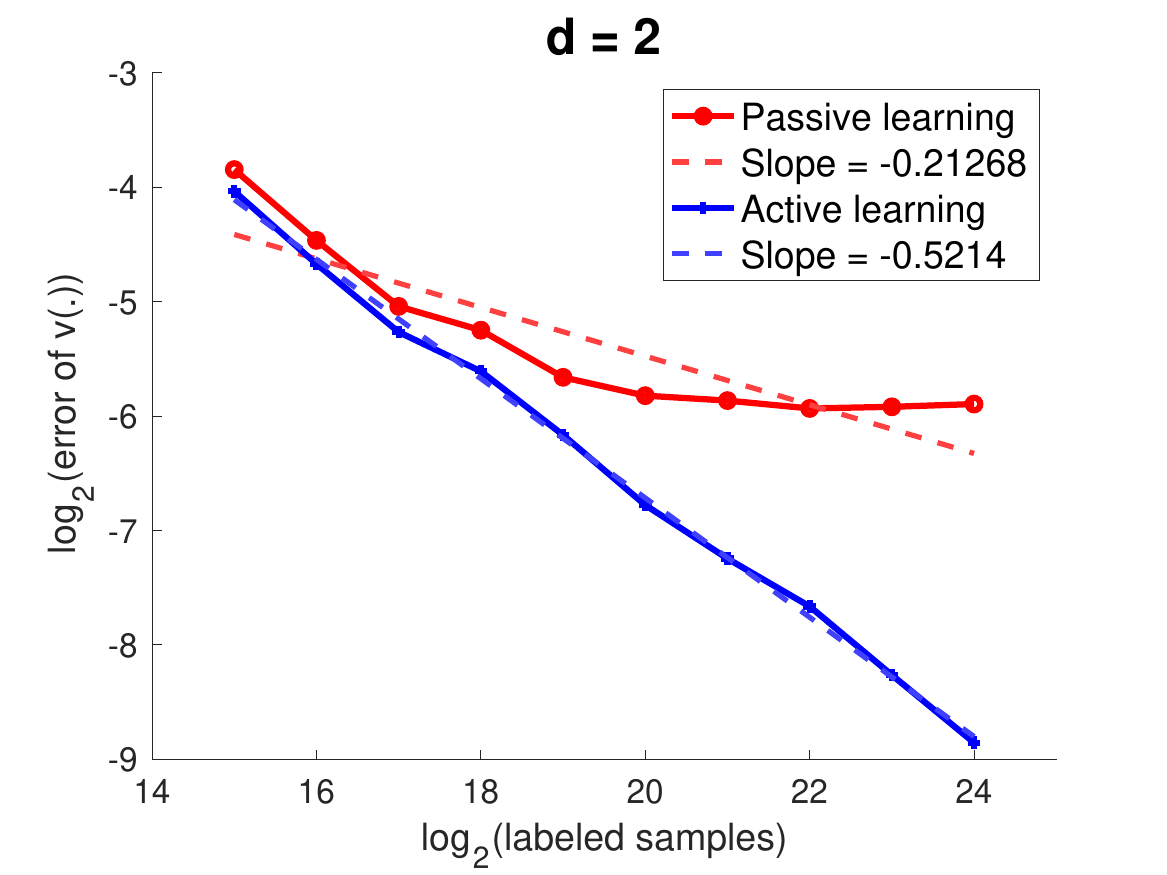}
\includegraphics[width=0.32\textwidth]{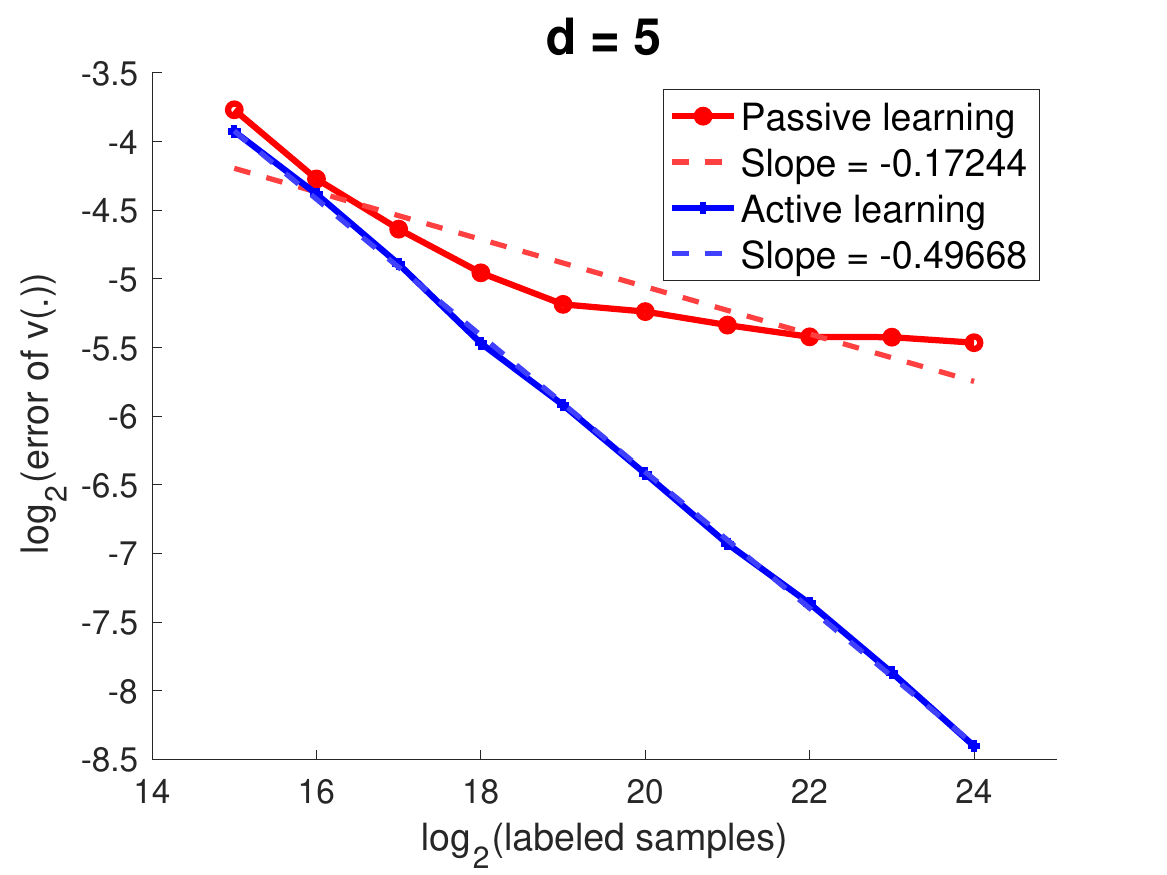}
\includegraphics[width=0.32\textwidth]{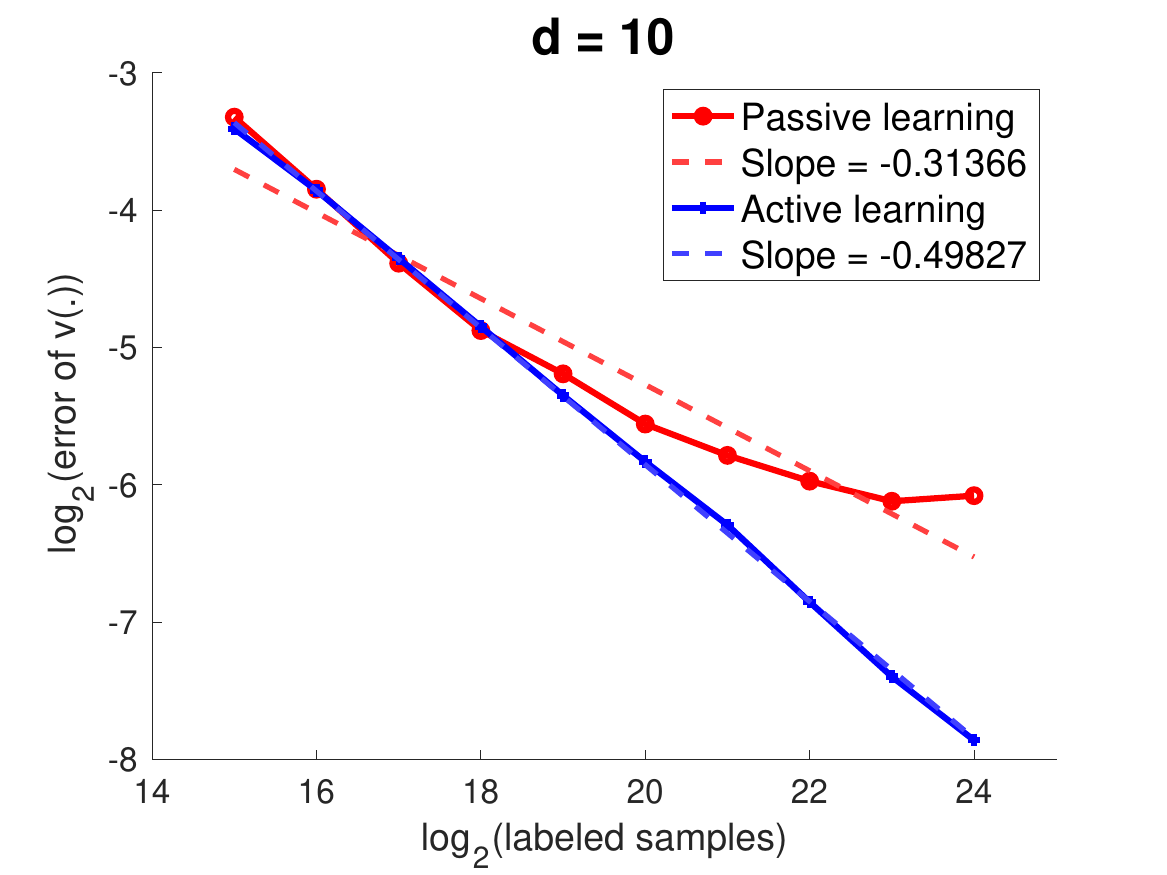}
\caption{Log-log plot of the estimation errors of $v(\cdot)$ as a function of the number of labeled (queried) samples $n$, for $d\in\{2,3,5\}$.
The dotted lines are fitted linear regression of the log-log plots.}
\label{fig:exp-convergence}
\end{figure}

Figure \ref{fig:exp-convergence} reports the estimation errors of the active learning algorithm and the passive learning baseline
for dimension settings of $d\in\{2,5,10\}$. 
Each reported error statistic is averaged over 200 independent trials, since both the labels and algorithm decisions contain randomness.
As we can see, our proposed active learning algorithm (the blue curves) outperforms significantly the estimates of the baseline passive learning
algorithm (the red curves), demonstrating the sample efficiency of active learning.

We further fit linear regression models on the log-log plots for both algorithms.
For the active learning algorithm, the slopes of the fitted linear models are very close to $-0.5$, suggesting an asymptotic convergence rate
of $n(\varepsilon,\delta)\asymp 1/\varepsilon^2$. This matches our theoretical results established in Theorem \ref{thm:meta-algorithm}.
On the other hand, the slopes of fitted models for the passive learning baseline range from $-0.17$ to $-0.31$, which are orders of magnitudes slower
convergence rates compared to the $1/\varepsilon^2$ rates for active learning methods.

\subsection{Sensitivity of model dimensions}

We use numerical results to evaluate the sensitivity of estimation errors with respect to the dimensions of the underlying linear model $d$.
In Figure \ref{fig:exp-dims}, we report the estimation errors of the active learning algorithm and the passive learning baseline
for dimensions $d$ ranging from 3 to 30.

\begin{figure}[t]
\centering
\includegraphics[width=0.5\textwidth]{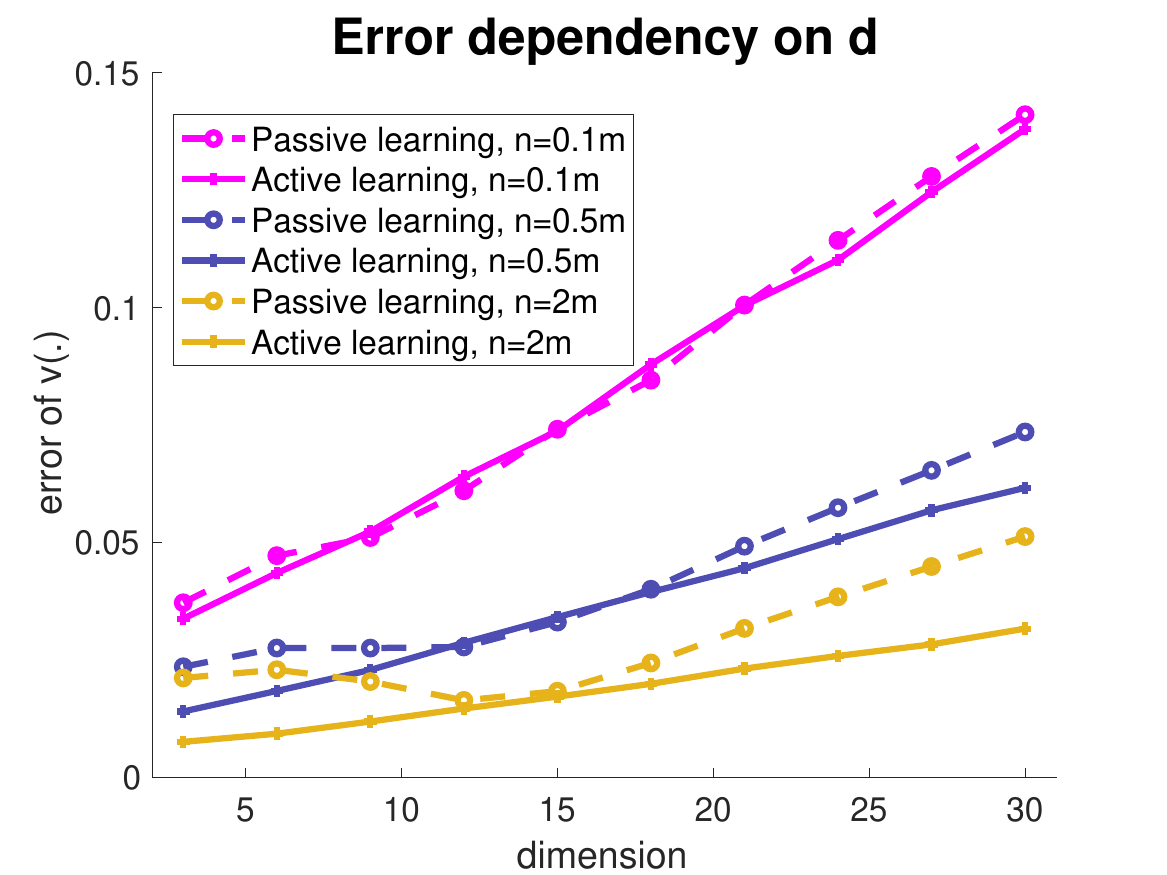}
\caption{Plot of the estimation errors of $v(\cdot)$ for different labeled (queried) samples $n$ and $d\in[3,30]$ settings. Here $m$ in the legend stands for million (e.g., $n=2m$ means the sample size is 2 million).}
\label{fig:exp-dims}
\end{figure}

As we can see in Figure \ref{fig:exp-dims}, the estimation errors of our proposed active learning approach scale near linearly with the dimension $d$
of the underlying linear model. The active learning algorithm also consistently outperforms the passive learning baseline,
especially in large $n$ or $d$ settings.

\subsection{Sensitivity of unlabeled samples}

\begin{figure}[t]
\centering
\includegraphics[width=0.45\textwidth]{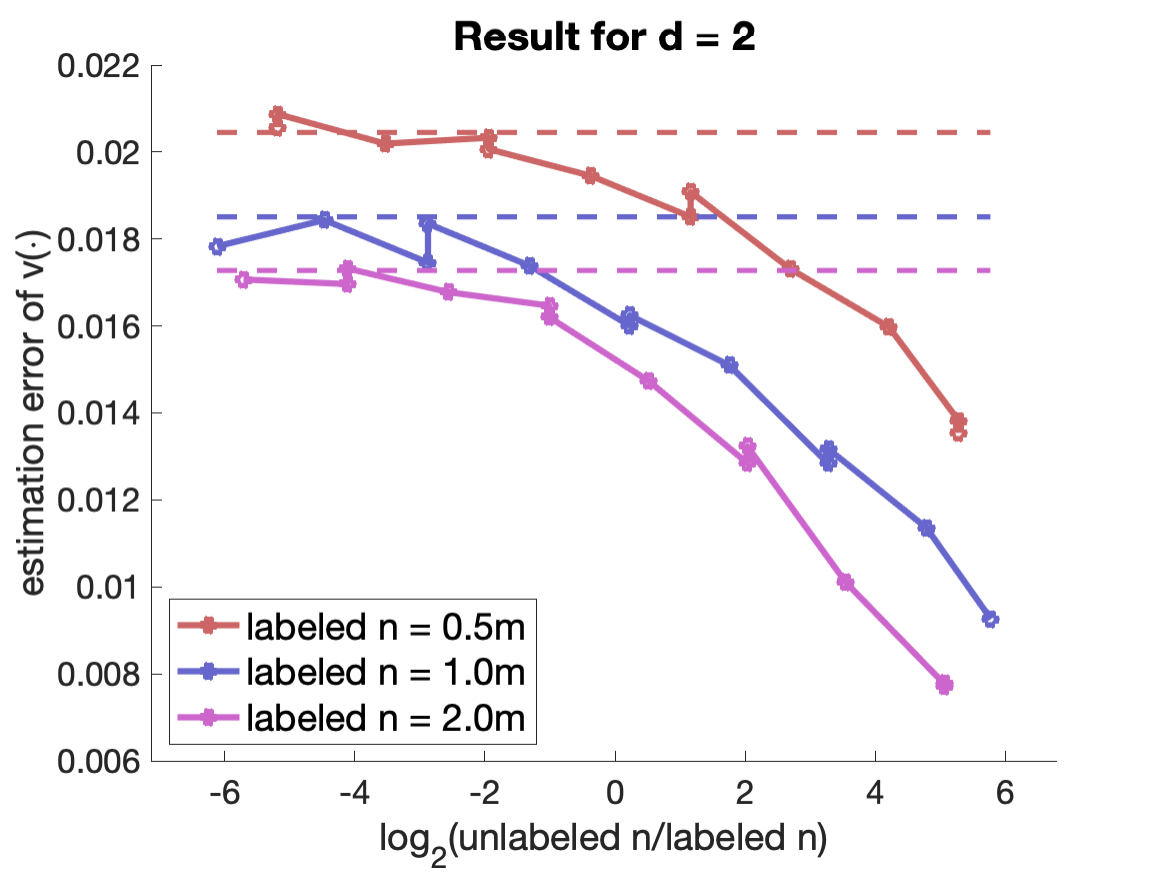}
\includegraphics[width=0.45\textwidth]{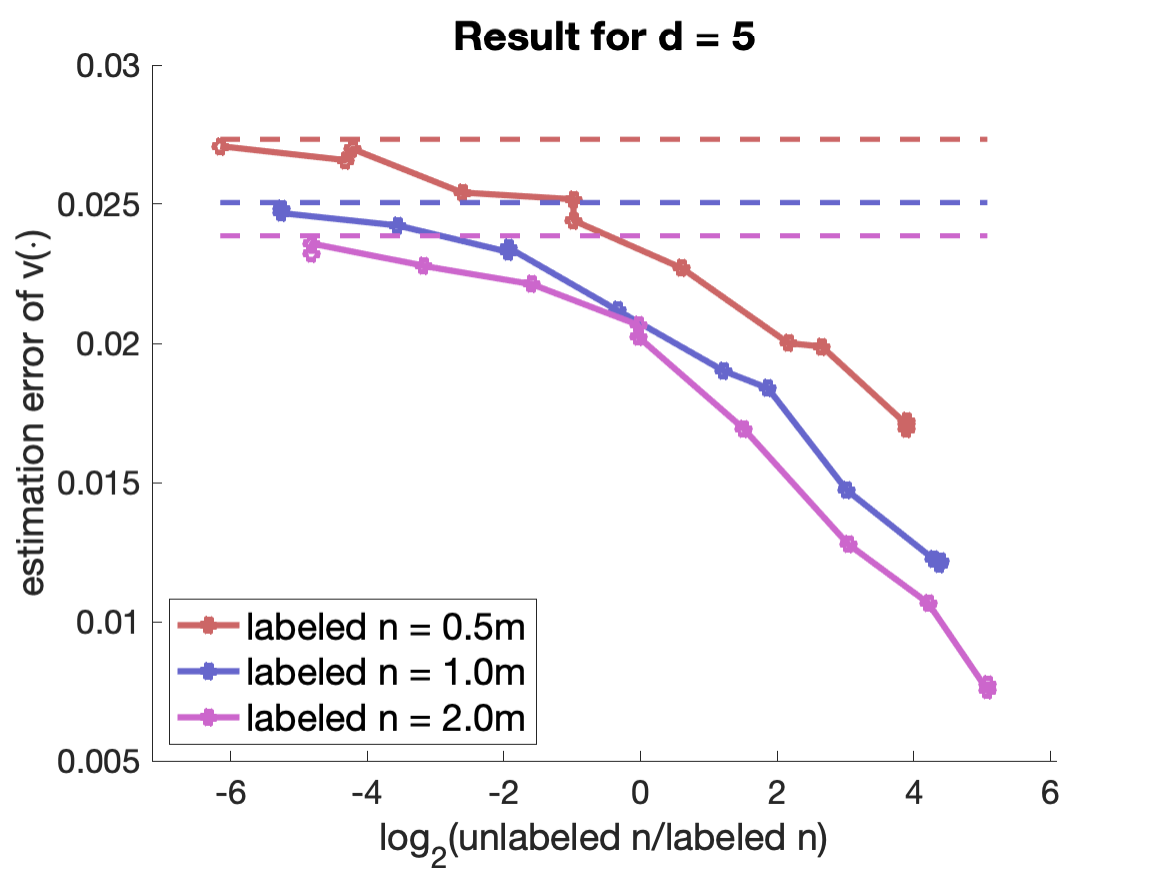}
\caption{Plots of estimation errors of $v(\cdot)$ by the active learning algorithm with different budgets of unlabeled samples (i.e., the number of labeled samples $n(\varepsilon, \delta)$  varies from 0.5 million, 1 million, to 2 million). The $x$-axis is the $\log$ of the ratio between the number of unlabeled samples and the number of labeled samples. The $y$-axis is the estimation error. 
Dotted lines are errors of the passive learning algorithm.}
\label{fig:sensitivity-unlabeled}
\end{figure}

In this section we report numerical results showing how the estimation errors of the proposed active learning algorithm
decrease as the algorithm has access to more unlabeled samples.
For this purpose, we define the \emph{ratio} between unlabeled and labeled samples as
$$
\rho := \frac{m(\varepsilon,\delta)-n(\varepsilon,\delta)}{n(\varepsilon,\delta)},
$$
where $m(\varepsilon,\delta)$ is the total number of samples consumed and $n(\varepsilon,\delta)$ is the number of samples that are  labeled/queried. Thus, the numerator  $m(\varepsilon,\delta)-n(\varepsilon,\delta)$ is the number of skipped samples.
While our theoretical results in Theorem \ref{thm:meta-algorithm} indicate that $\rho$ has to scale as large as $O(1/\varepsilon)$ for the 
margin-based active learning algorithm.
In practice, however, it is possible to achieve significant estimation accuracy improvements with smaller values of $\rho$:
when the unlabeled sample budget is completely consumed, the margin-based active learning algorithm will revert back to passive learning
without any additional sample selection being carried out.

In Figure \ref{fig:sensitivity-unlabeled} we plot the estimation errors of the utility $v(\cdot)$ as a function of $\log_2(\rho)$,
with larger values of $\log_2(\rho)$ indicating more unlabeled samples involved.
We also report the estimation errors of the passive learning algorithm as a benchmark, which can be regarded as an instance of $\rho=0$ (i.e., all the samples are labeled).  
As we can see, the estimation errors of our proposed active learning algorithm decrease rapidly with increasing $\rho$,
and the performance increase is significant when $\rho$ is as small as 0.5 or 1.0.
This shows that even with a modest amount of unlabeled samples, the active learning procedure can already significantly increase
the accuracy of the estimated utility function $\hat v(\cdot)$.

\section{Conclusions}
\label{sec:con}

In this paper, we study a learning problem in contextual search, where the goal is to use as fewer queries as possible to accurately estimate the mean value function. To this end, we propose a margin-based active learning algorithm with tri-section search scheme and establish the corresponding PAC learning sample complexity bound. Our bound shows a significant improvement over the passive setting.

There are several interesting future directions. First, we assume a linear model in this paper. It would be interesting to extend the linear model to more general parametric and non-parametric models.  Second, in general, establishing lower bound result in active learning for binary feedback is very challenging. Despite that, it is worth to explore the optimality of our algorithm. Third, we hope the proposed active learning algorithm would inspire more research on adaption of active learning to solve important operations problems. 

\section*{Acknowledgment}

The authors thank the department editor, the associated editor, and the anonymous referees for many useful suggestions and feedback, which greatly improves the paper. Xi Chen and Quanquan Liu would like to thank the support from NSF via the Grant IIS-1845444.

\begin{appendices}

\section{Some technical lemmas}

\begin{lemma}
Suppose $x\sim P_U$. Then for any measurable set $A\subseteq[-1,1]$, it holds that
$\Pr[x_1\in A]\leq \sqrt{\frac{d+1}{2\pi}}\int_{u\in A} e^{-(d-1)u^2/2}\ud u$.
If $A\subseteq[-1/\sqrt{2},1/\sqrt{2}]$ then $\Pr[x_1\in A]\geq \sqrt{\frac{d+1}{16\pi}}\int_{u\in A}e^{-(d-1)u^2/2}\ud u$.
\label{lem:marginal-d1}
\end{lemma}
\begin{proof}{Proof of Lemma \ref{lem:marginal-d1}.}
Let $V_d=\pi^{d/2}/\Gamma(1+d/2)$ be the volume of $\mathbb B_2(d)$, where $\Gamma(z)=\int_0^{\infty}x^{z-1}e^{-x}\ud x$
is the Gamma function.
Using change-of-variable in multivariate integration, it is easy to verify that $\Pr[x_1\in A] = \frac{V_{d-1}}{V_d}\int_{u\in A}(1-u^2)^{(d-1)/2}\ud u
=\frac{1}{\sqrt{\pi}}\frac{\Gamma(d/2+1)}{\Gamma(d/2+1/2)}\int_{u\in A}(1-u^2)^{(d-1)/2}\ud u$.
By Kershaw's inequality \citep{kershaw1983some}, for any $m>0$ it holds that $\frac{m}{\sqrt{m+1}}<\frac{\Gamma(m+1/2)}{\Gamma(m)}<\sqrt{m}$.
Subsequently, 
$$
\Pr[x_1\in A]\leq \frac{\sqrt{d/2+1/2}}{\sqrt{\pi}}\int_{u\in A}(1-u^2)^{(d-1)/2}\ud u \leq \sqrt{\frac{d+1}{2\pi}}\int_{u\in A}e^{-(d-1)u^2/2}\ud u,
$$
where the last inequality holds because $1-z\leq e^{-z}$ for all $u\geq 0$.
For the other direction, note that $m/\sqrt{m+1}\geq 1/\sqrt{2}$ for $m\geq 1$ and $1-z\geq 0.5e^{-z}$ for all $0\leq z\leq 1/2$. We have
$$
\Pr[x_1\in A]\geq \frac{\sqrt{d/2+1/2}}{\sqrt{2\pi}}\int_{u\in A}(1-u^2)^{(d-1)/2}\ud u \geq \sqrt{\frac{d+1}{16\pi}}\int_{u\in A}e^{-(d-1)u^2}\ud u,
$$
where the last inequality holds because $u^2\leq 1/2$ for all $u\in A$ as assumed. $\square$
\end{proof}

\begin{lemma}
Suppose $d\geq 2$ and $x\sim P_U$. Then for any measurable set $A\subseteq [-1,1]^2\cap\mathbb B_2(d)$, it holds that
$\Pr[(x_1,x_2)\in A]\leq \frac{d}{2\pi}\int_{(u_1,u_2)\in A}e^{-(d-2)(u_1^2+u_2^2)/2}\ud u_1\ud u_2$.
If $x_1^2+x_2^2\leq 1/2$ for all $(x_1,x_2)\in A$, then $\Pr[(x_1,x_2)\in A]\geq \frac{d}{4\pi}\int_{(u_1,u_2)\in A}e^{-(d-2)(u_1^2+u_2^2)/2}\ud u_1\ud u_2$.
\label{lem:marginal-d2}
\end{lemma}
\begin{proof}{Proof of Lemma \ref{lem:marginal-d2}.}
By the change-of-variable formula, $\Pr[(x_1,x_2)\in A] = \frac{V_{d-2}}{V_d}\int_{(u_1,u_2)\in A}(1-u_1^2-u_2^2)^{(d-2)/2}\ud u_1\ud u_2 = \frac{d}{2\pi}\int_{(u_1,u_2)\in A}(1-u_1^2-u_2^2)^{(d-2)/2}\ud u_1\ud u_2$. 
The rest of the proof is identical to the proof of Lemma \ref{lem:marginal-d1}. $\square$
\end{proof}

We next define some useful notations that will make our proof similar.
For any $x\in\mathbb B_2(d)$ and $b\in\mathbb R$, define 
\begin{equation}
\eta_b(x) := \Pr[y=1|x,b], \;\;\;\;\;\;\Delta_b(x) := v(x)-b = \langle x,w^*\rangle + \mu^*-b.
\label{eq:defn-eta-delta}
\end{equation}
Because $y=1$ if and only if $v(x)+\xi\geq b$ and $\xi\sim P_\xi$ with $\int_{-\infty}^0 f_\xi(u)\ud u=\int_0^\infty f_\xi(u)\ud u=1/2$ (see Assumption (A3)), we have that
$\eta_b(x)-\frac{1}{2} = \int_{-\Delta_b(x)}^0 f_\xi(u)\ud u= F_\xi(0)-F_\xi(-\Delta_b(x))$ if $\Delta_b(x)\geq 0$, and $\eta_b(x)-\frac{1}{2}=-\int_0^{-\Delta_b(x)}f_\xi(u)\ud u=F_\xi(0)-F_\xi(-\Delta_b(x))$ if $\Delta_b(x)<0$,
where $F_\xi(\cdot)$ is the CDF of $P_\xi$.
Since $\eta_b(x)$ only depends on $\Delta_b(x)$, we can define 
\begin{equation}
\phi(\Delta) := F_\xi(0) - F_\xi(-\Delta).
\label{eq:defn-phi}
\end{equation}
It then holds that $\eta_b(x) - \frac{1}{2} = \phi(\Delta_b(x))$.
Furthermore, by definition we have that $\phi(\Delta)\leq 0$ for all $\Delta\leq 0$, $\phi(\Delta)\geq 0$ for all $\Delta\geq 0$,
$\phi(0)=0$ and $\phi'(\Delta) = f_\xi(-\Delta)\in [c_\xi,C_\xi]$ for all $|\Delta|\leq 2$, thanks to Assumption (A3).

Now let $\mathcal V=\{v(\cdot): v(\cdot)=\langle\cdot, w\rangle -\beta, w\in\mathbb R^d,\beta\in\mathbb R\}$ be a hypothesis class of 
non-homogeneous $d$-dimensional linear classifiers.
The following lemma is a consequence of the classical VC theory of classification (see, e.g., \cite[Theorem 8]{balcan2007margin}).
\begin{lemma}
Fix a distribution $P$ supported on $\mathbb B_2(d)$ and a joint distribution $Q$ supported on $\mathbb B_2(d)\times \{0,1\}$,
such that the marginal of $Q$ on $\mathbb B_2(d)$ is $P$.
Let $v^*=\arg\min_{v\in\mathcal V}\err(v|Q)$, where $\err(v|Q) = \Pr_{(x,y)\sim Q}[y\neq \sgn(v(x))]$.
Let $\{(x_i,y_i)\}_{i=1}^n\overset{i.i.d.}{\sim} Q$ be $n$ i.i.d.~samples, and $\hat v=\arg\min_{v\in\mathcal V}\sum_{i=1}^n \vct 1\{y_i\neq v(x_i)\}$ be the empirical 
risk minimizer. Then there exists a universal constant $C>0$ such that for any $\epsilon,\delta\in(0,1)$, 
if $n\geq C\epsilon^{-2}(d+1+\ln(1/\delta))$ then it holds with probability $1-\delta$ that
$\err(\hat v|Q)-\err(v^*|Q)\leq 2\epsilon$.
\label{lem:VC-erm}
\end{lemma}
\begin{proof}{Proof of Lemma \ref{lem:VC-erm}.}
Note that the VC dimension of $\mathcal V$ is $d+1$.
By \cite[Theorem 8]{balcan2007margin}, it holds with probability $1-\delta$ that
$$
\Pr\left[\forall v\in\mathcal V, \left|\frac{1}{n}\sum_{i=1}^n\vct 1\{y_i\neq\sgn(v(x_i))\} - \err(v|Q)\right|\leq\epsilon\right]\geq 1-\delta.
$$
The lemma is proved by using triangle inequality. $\square$
\end{proof}

\end{appendices}


\bibliographystyle{ormsv080}
\bibliography{refs}

\begin{thebibliography}{37}
\expandafter\ifx\csname natexlab\endcsname\relax\def\natexlab#1{#1}\fi
\expandafter\ifx\csname url\endcsname\relax
  \def\url#1{{\tt #1}}\fi
\expandafter\ifx\csname urlprefix\endcsname\relax\def\urlprefix{URL }\fi
\expandafter\ifx\csname urlstyle\endcsname\relax
  \expandafter\ifx\csname doi\endcsname\relax
  \def\doi#1{doi:\discretionary{}{}{}#1}\fi \else
  \expandafter\ifx\csname doi\endcsname\relax
  \def\doi{doi:\discretionary{}{}{}\begingroup \urlstyle{rm}\Url}\fi \fi

\bibitem[{Albert(1961)}]{albert1961sequential}
Albert, Arthur~E. 1961.
\newblock The sequential design of experiments for infinitely many states of
  nature.
\newblock {\it The Annals of Mathematical Statistics\/}  774--799.

\bibitem[{Araman and Caldentey(2021)}]{araman2021diffusion}
Araman, Victor~F, Rene Caldentey. 2021.
\newblock Diffusion approximations for a class of sequentialexperimentation
  problems.
\newblock {\it Management Science (Articles in Advance)\/}  1--22.

\bibitem[{Awasthi et~al.(2017)Awasthi, Balcan, and Long}]{awasthi2017power}
Awasthi, Pranjal, Maria~Florina Balcan, Philip~M Long. 2017.
\newblock The power of localization for efficiently learning linear separators
  with noise.
\newblock {\it Journal of the ACM\/} {\bf 63}(6) 1--27.

\bibitem[{Balcan et~al.(2009)Balcan, Beygelzimer, and
  Langford}]{balcan2009agnostic}
Balcan, Maria-Florina, Alina Beygelzimer, John Langford. 2009.
\newblock Agnostic active learning.
\newblock {\it Journal of Computer and System Sciences\/} {\bf 75}(1) 78--89.

\bibitem[{Balcan et~al.(2007)Balcan, Broder, and Zhang}]{balcan2007margin}
Balcan, Maria-Florina, Andrei Broder, Tong Zhang. 2007.
\newblock Margin based active learning.
\newblock Nader~H. Bshouty, Claudio Gentile, eds., {\it Proceedings of the 20th
  Annual Conference on Learning Theory\/}.

\bibitem[{Balcan and Long(2013)}]{balcan2013active}
Balcan, Maria-Florina, Phil Long. 2013.
\newblock Active and passive learning of linear separators under log-concave
  distributions.
\newblock Shai Shalev-Shwartz, Ingo Steinwart, eds., {\it Proceedings of the
  26th Annual Conference on Learning Theory\/}, vol.~30. 288--316.

\bibitem[{Balcan and Zhang(2017)}]{balcan2017sample}
Balcan, Maria-Florina, Hongyang Zhang. 2017.
\newblock Sample and computationally efficient learning algorithms under
  s-concave distributions.
\newblock {\it Advances in Neural Information Processing Systems\/}.

\bibitem[{Bastani and Bayati(2020)}]{bastani2020online}
Bastani, Hamsa, Mohsen Bayati. 2020.
\newblock Online decision making with high-dimensional covariates.
\newblock {\it Operations Research\/} {\bf 68}(1) 276--294.

\bibitem[{Bastani et~al.(2021)Bastani, Bayati, and
  Khosravi}]{bastani2021mostly}
Bastani, Hamsa, Mohsen Bayati, Khashayar Khosravi. 2021.
\newblock Mostly exploration-free algorithms for contextual bandits.
\newblock {\it Management Science\/} {\bf 67}(3) 1329--1349.

\bibitem[{Ben-David and Urner(2014)}]{ben2014sample}
Ben-David, Shai, Ruth Urner. 2014.
\newblock The sample complexity of agnostic learning under deterministic
  labels.
\newblock Maria~Florina Balcan, Vitaly Feldman, Csaba Szepesvári, eds., {\it
  Proceedings of The 27th Annual Conference on Learning Theory\/}, vol.~35.
  527--542.

\bibitem[{Chen et~al.(2022)Chen, Chen, and Li}]{Chen:22:asym}
Chen, Xi, Yunxiao Chen, Xiaoou Li. 2022.
\newblock Asymptotically optimal sequential design for rank aggregation.
\newblock {\it Mathematics of Operations Research (Articles in Advance)\/} .

\bibitem[{Chernoff(1959)}]{chernoff1959sequential}
Chernoff, Herman. 1959.
\newblock Sequential design of experiments.
\newblock {\it The Annals of Mathematical Statistics\/} {\bf 30}(3) 755--770.

\bibitem[{Cohen et~al.(2020)Cohen, Lobel, and Leme}]{Cohen:20:feature}
Cohen, Maxime~C., Ilan Lobel, Renato~Paes Leme. 2020.
\newblock Feature-based dynamic pricing.
\newblock {\it Management Science\/} {\bf 66}(11) 4921--4943.

\bibitem[{Cohn et~al.(1994)Cohn, Atlas, and Ladner}]{cohn1994improving}
Cohn, David, Les Atlas, Richard Ladner. 1994.
\newblock Improving generalization with active learning.
\newblock {\it Machine Learning\/} {\bf 15}(2) 201--221.

\bibitem[{Cohn(1996)}]{cohn1996neural}
Cohn, David~A. 1996.
\newblock Neural network exploration using optimal experiment design.
\newblock {\it Neural Networks\/} {\bf 9}(6) 1071--1083.

\bibitem[{Dasgupta(2005{\natexlab{a}})}]{dasgupta2005analysis}
Dasgupta, Sanjoy. 2005{\natexlab{a}}.
\newblock Analysis of a greedy active learning strategy.
\newblock {\it Advances in Neural Information Processing Systems\/}.

\bibitem[{Dasgupta(2005{\natexlab{b}})}]{dasgupta2005coarse}
Dasgupta, Sanjoy. 2005{\natexlab{b}}.
\newblock Coarse sample complexity bounds for active learning.
\newblock {\it Advances in Neural Information Processing Systems\/}.

\bibitem[{Elfving(1952)}]{elfving1952optimum}
Elfving, Gustav. 1952.
\newblock Optimum allocation in linear regression theory.
\newblock {\it The Annals of Mathematical Statistics\/}  255--262.

\bibitem[{Feng et~al.(2022)Feng, Caldentey, and Ryan}]{feng2021robust}
Feng, Yifan, Rene Caldentey, Christopher~Thomas Ryan. 2022.
\newblock Robust learning of consumer preferences.
\newblock {\it Operations Research\/} {\bf 70}(2) 918--962.

\bibitem[{Hanneke(2007)}]{hanneke2007bound}
Hanneke, Steve. 2007.
\newblock A bound on the label complexity of agnostic active learning.
\newblock {\it Proceedings of the 24th International Conference on Machine
  Learning\/}. 353--360.

\bibitem[{Hanneke et~al.(2014)}]{hanneke2014theory}
Hanneke, Steve, et~al. 2014.
\newblock Theory of disagreement-based active learning.
\newblock {\it Foundations and Trends{\textregistered} in Machine Learning\/}
  {\bf 7}(2-3) 131--309.

\bibitem[{Kershaw(1983)}]{kershaw1983some}
Kershaw, D. 1983.
\newblock Some extensions of w. gautschi’s inequalities for the gamma
  function.
\newblock {\it Mathematics of Computation\/} {\bf 41}(164) 607--611.

\bibitem[{Krishnamurthy et~al.(2021)Krishnamurthy, Lykouris, Podimata, and
  Schapire}]{Krishnamurthy21Contextual}
Krishnamurthy, Akshay, Thodoris Lykouris, Chara Podimata, Robert~E. Schapire.
  2021.
\newblock Contextual search in the presence of irrational agents.
\newblock {\it Proceedings of the Symposium on Theory of Computing (STOC)\/}.

\bibitem[{Leme and Schneider(2018)}]{Leme:18:contextual}
Leme, R.~Paes, J.~Schneider. 2018.
\newblock Contextual search via intrinsic volumes.
\newblock {\it Proceedings of the IEEE Symposium on Foundations of Computer
  Science\/}.

\bibitem[{Li et~al.(2021)Li, Chen, Chen, Liu, and Ying}]{li2017optimal}
Li, Xiaoou, Yunxiao Chen, Xi~Chen, Jingchen Liu, Zhiliang Ying. 2021.
\newblock Optimal stopping and worker selection in crowdsourcing: An adaptive
  sequential probability ratio test framework.
\newblock {\it Statistica Sinica\/} {\bf 31} 519--546.

\bibitem[{Lobel et~al.(2018)Lobel, Leme, and Vladu}]{Lobel:17:Multidimensional}
Lobel, Ilan, Renato~Paes Leme, Adrian Vladu. 2018.
\newblock Multidimensional binary search for contextual decision-making.
\newblock {\it Operations Research\/} {\bf 66}(5) 1346--1361.

\bibitem[{Mammen and Tsybakov(1999)}]{mammen1999smooth}
Mammen, Enno, Alexandre~B Tsybakov. 1999.
\newblock Smooth discrimination analysis.
\newblock {\it The Annals of Statistics\/} {\bf 27}(6) 1808--1829.

\bibitem[{Naghshvar and Javidi(2013)}]{naghshvar2013active}
Naghshvar, Mohammad, Tara Javidi. 2013.
\newblock Active sequential hypothesis testing.
\newblock {\it The Annals of Statistics\/} {\bf 41}(6) 2703--2738.

\bibitem[{Settles(2012)}]{Settles:12:active}
Settles, B. 2012.
\newblock {\it Active Learning\/}.
\newblock Morgan \& Claypool.

\bibitem[{Shiffrin and Nosofsky(1994)}]{Shiffrin:94}
Shiffrin, R.~M., R.~M. Nosofsky. 1994.
\newblock Seven plus or minus two: A commentary on capacity limitations.
\newblock {\it Psychological Review\/} {\bf 101}(357--361).

\bibitem[{Stewart et~al.(2005)Stewart, Brown, and Chater}]{Stewart:05}
Stewart, Neil, Gordon~DA Brown, Nick Chater. 2005.
\newblock Absolute identification by relative judgment.
\newblock {\it Psychological review\/} {\bf 112}(4) 881--911.

\bibitem[{Vapnik(2013)}]{vapnik2013nature}
Vapnik, Vladimir. 2013.
\newblock {\it The nature of statistical learning theory\/}.
\newblock Springer science \& business media.

\bibitem[{Vapnik and Chervonenkis(2015)}]{vapnik2015uniform}
Vapnik, Vladimir~N, A~Ya Chervonenkis. 2015.
\newblock On the uniform convergence of relative frequencies of events to their
  probabilities.
\newblock {\it Measures of complexity\/}. Springer, 11--30.

\bibitem[{Wager and Xu(2021)}]{wager2021diffusion}
Wager, Stefan, Kuang Xu. 2021.
\newblock Diffusion asymptotics for sequential experiments.
\newblock {\it arXiv preprint arXiv:2101.09855\/} .

\bibitem[{Wang and Singh(2016)}]{wang2016noise}
Wang, Yining, Aarti Singh. 2016.
\newblock Noise-adaptive margin-based active learning and lower bounds under
  tsybakov noise condition.
\newblock {\it Thirtieth AAAI Conference on Artificial Intelligence\/}.

\bibitem[{Wang and Zenios(2020)}]{wang2020adaptive}
Wang, Zhengli, Stefanos Zenios. 2020.
\newblock Adaptive design of clinical trials: A sequential learning approach.
\newblock {\it Available at SSRN 3713924\/} .

\bibitem[{Zhang and Chaudhuri(2014)}]{zhang2014beyond}
Zhang, Chicheng, Kamalika Chaudhuri. 2014.
\newblock Beyond disagreement-based agnostic active learning.
\newblock {\it Advances in Neural Information Processing Systems\/}.

\end{thebibliography}

\end{document}